\def\final{1}
\def\short{0}
\def\submission{0}

\ifnum\short=0

\else

\fi

\documentclass[11pt]{article}

\usepackage{algorithm, algorithmic}
\usepackage{amsmath, amssymb, amsthm}
\usepackage[inline]{enumitem}
\usepackage{framed}
\usepackage{verbatim}
\usepackage{color}
\usepackage{microtype}
\usepackage{bbm}
\ifnum\short=0 
	\ifnum\submission=0
                \usepackage{kpfonts}
                \DeclareMathAlphabet{\mathsf}{OT1}{cmss}{m}{n}
                \SetMathAlphabet{\mathsf}{bold}{OT1}{cmss}{bx}{n}
                
                \usepackage[margin=1.01in]{geometry}
	\else
		\usepackage{times}
		\usepackage[margin=1.0in]{geometry}
	\fi
\else
	\usepackage{times}
	\usepackage[margin=1.0in]{geometry}
\fi 
\usepackage{multirow, tabularx}
\definecolor{DarkGreen}{rgb}{0.2,0.6,0.2}
\definecolor{DarkRed}{rgb}{0.6,0.2,0.2}
\definecolor{DarkBlue}{rgb}{0.15,0.15,0.55}
\definecolor{DarkPurple}{rgb}{0.4,0.2,0.4}

\usepackage[pdftex]{hyperref}
\hypersetup{
    linktocpage=true,
    colorlinks=true,				
    linkcolor=DarkBlue,		
    citecolor=DarkBlue,	
    urlcolor=DarkBlue,		
}

\ifnum\final = 0
\newcommand{\mynote}[1]{\marginpar{\tiny #1}}
\newcommand{\Bignote}[1]{{\tiny #1}}
\else
\newcommand{\mynote}[1]{}
\newcommand{\Bignote}[1]{}
\fi

\newcolumntype{Y}{>{\centering\arraybackslash}X}

\newcommand{\ex}[2]{\underset{#1}{\mathbb{E}}\left[ #2 \right]}
\newcommand{\exx}[1]{\underset{#1}{\mathbb{E}}{\;\;\;}}

\newcommand{\poly}{\mathrm{poly}}

\newcommand{\from}{:}

\newcommand{\eps}{\varepsilon}

\newcommand{\R}{\mathbb{R}}

\newcommand{\cO}{\mathcal{O}}

\newcommand{\1}{\mathbbm{1}}

\def\E{\operatorname*{\mathbb{E}}}
\newcommand{\AAA}{\mathcal A}
\newcommand{\BBB}{\mathcal B}

\newcommand{\DDD}{\mathcal D}
\newcommand{\FFF}{\mathcal F}

\newcommand{\UUU}{\mathcal U}

\newtheorem{theorem}{Theorem}[section]

\newtheorem{lem}[theorem]{Lemma}

\newtheorem{claim}[theorem]{Claim}

\theoremstyle{definition}

\newtheorem{definition}[theorem]{Definition}

\newcommand{\univ}{X}
\newcommand{\dist}{\mathcal{D}}

\newcommand{\query}{f}

\newcommand{\querysetDelta}{\FFF_{\lambda}}

\newcommand{\samp}{x}
\newcommand{\sample}{S}
\newcommand{\trials}{T}
\newcommand{\trial}{t}
\newcommand{\samples}{\vec{S}}

\makeatletter
\newcommand{\oset}[3][0ex]{%
  \mathrel{\mathop{#3}\limits^{
    \vbox to#1{\kern-2\ex@
    \hbox{$\scriptscriptstyle#2$}\vss}}}}
\makeatother

\def\presuper#1#2%
  {\mathop{}%
   \mathopen{\vphantom{#2}}^{\scriptscriptstyle #1}%
   \kern-5\scriptspace%
   #2}

\begin{document}

\begin{titlepage}
 
\title{Concentration Bounds for High Sensitivity Functions Through Differential Privacy\thanks{Research by K.N.\ and U.S.\ is supported by NSF grant No.\  1565387.}}

\author{
Kobbi Nissim\thanks{Dept.\ of Computer Science, Georgetown University {\em and} Center for Research on Computation and Society (CRCS), Harvard University. {\tt kobbi.nissim@georgetown.edu}.}
\and 
Uri Stemmer\thanks{Center for Research on Computation and Society (CRCS), Harvard University.  {\tt stemmer@cs.bgu.ac.il}.}
}

\date{\today}
\maketitle
\setcounter{page}{0} \thispagestyle{empty}

\begin{abstract}
A new line of work~\cite{DworkFHPRR15,HU14,SU15,BassilyNSSSU16} demonstrates how differential privacy~\cite{DMNS06} can be used as a mathematical tool for guaranteeing generalization in adaptive data analysis. Specifically, if a differentially private analysis is applied on a sample $S$ of i.i.d.\ examples to select a low-sensitivity function $f$, then w.h.p.\ $f(S)$ is close to its expectation, although $f$ is being chosen based on the data. 

Very recently, Steinke and Ullman~\cite{SU17} observed that these generalization guarantees can be used for proving concentration bounds in the non-adaptive setting, where the low-sensitivity function is fixed beforehand. In particular, they obtain alternative proofs for classical concentration bounds for low-sensitivity functions, such as the Chernoff bound and McDiarmid's Inequality. 

In this work, we set out to examine the situation for functions with {\em high}-sensitivity, for which differential privacy does not imply generalization guarantees under adaptive analysis. We show that differential privacy can be used to prove concentration bounds for such functions in the non-adaptive setting.
\end{abstract}

\vfill
\vfill

\end{titlepage}

\section{Introduction}

A new line of work~\cite{DworkFHPRR15,HU14,SU15,BassilyNSSSU16} demonstrates how differential privacy~\cite{DMNS06} can be used as a mathematical tool for guaranteeing statistical validity in data analysis. Specifically, if a differentially private analysis is applied on a sample $S$ of i.i.d.\ examples to select a low-sensitivity function $f$, then w.h.p.\ $f(S)$ is close to its expectation, even when $f$ is being chosen based on the data.  Dwork et al.~\cite{DworkFHPRR15} showed how to utilize this connection for the task of answering {\em adaptively chosen} queries w.r.t.\ an unknown distribution using i.i.d.\ samples from it.

To make the setting concrete, consider a data analyst interested in learning properties of an unknown distribution $\DDD$. The analyst interacts with the distribution $\DDD$ via a {\em data curator} $\AAA$ holding a database $S$ containing $n$ i.i.d.\ samples from $\DDD$. The interaction is adaptive, where at every round the analyst specifies a query $q:X^n\rightarrow\R$ and receives an answer $a_q(S)$ that (hopefully) approximates $q(\DDD^n) \triangleq \E_{S'\sim\DDD^n}[q(S')]$. 
As the analyst chooses its queries based on previous interactions with the data, we run the risk of overfitting if $\AAA$ simply answers every query with its empirical value on the sample $S$. However, if $\AAA$ is a differentially private algorithm then the interaction would not lead to overfitting: 
\begin{theorem}[\cite{DworkFHPRR15,BassilyNSSSU16}, informal]\label{thm:BNSSSU}
A function $f:X^n\rightarrow\R$ has sensitivity $\lambda$ if $|f(S)-f(S')| \leq \lambda$ for every pair $S, S' \in X^n$ differing in only one entry.
Define $\query(\DDD^n) \triangleq \ex{S'\sim\DDD^n}{\query(S')}$.
Let $\AAA : X^{n} \rightarrow \querysetDelta$ be $(\eps,\delta)$-differentially private where $\querysetDelta$ is the class of $\lambda$-sensitive functions, and $n\geq\frac{1}{\eps^2}\log(\frac{4\eps}{\delta})$. Then for every distribution $\DDD$ on $X$,
$$
\Pr_{\substack{\sample\sim\DDD^n\\f\leftarrow\AAA(S)}}\left[ \left| \query(\sample) - \query(\DDD^n) \right| \geq 18\eps\lambda n \right] < \frac{\delta}{\eps}.
$$
\end{theorem}

In words, if $\AAA$ is a differentially private algorithm operating on a database containing $n$ i.i.d.\ samples from the distribution $\DDD$, then $\AAA$ cannot (with significant probability) identify a low-sensitivity function that behaves differently on the sample $S$ and on $\DDD^n$.

Very recently, Steinke and Ullman~\cite{SU17} observed that Theorem~\ref{thm:BNSSSU} gives alternative proofs for classical concentration bounds for low-sensitivity functions, such as the Chernoff bound and McDiarmid's Inequality:  Fix a function $f:X^n\rightarrow\R$ with sensitivity $\lambda$ and consider the trivial mechanism $\AAA_f$ that ignores its input and always outputs $f$. Such a mechanism is $(\eps,\delta)$-differentially private for any choice of $\eps,\delta\geq 0$ and hence Theorem~\ref{thm:BNSSSU} yields (up to constants) McDiarmid's Inequality:  
\begin{equation}
\Pr_{\sample\sim\DDD^n}\left[ \left| f(S) - f(\DDD^n) \right| \geq 18\eps\lambda n \right] < \frac{\delta}{\eps}=2^{-\Omega(\eps^2\cdot n)},\label{eq:SU17}
\end{equation}
where the last equality follows by setting $n=\frac{1}{\eps^2}\log(\frac{4\eps}{\delta})$.

In light of this result it is natural to ask if similar techniques yield concentration bounds for more general families of queries, and in particular queries that are not low-sensitivity functions. In this work we derive conditions under which this is the case. 

\subsection{Differential Privacy, Max-Information, and Typical Stability}

Let $\DDD$ be a fixed distribution over a domain $X$, and consider a family of functions mapping databases in $X^n$ to the reals, such that for every function $f$ in the family  we have that $|f(S)-f(\DDD^n)|$ is small w.h.p.\ over $S\sim\DDD^n$. Specifically,
$$\FFF_{\alpha,\beta}(\DDD) = \left\{  \;\; f:X^n\rightarrow\R \;\;\; : \;\;\;  \Pr_{S\sim\DDD^n}[|f(S)-f(\DDD^n)|>\alpha]\leq\beta \;\; \right\}.$$
That is, for every function $f\in \FFF_{\alpha,\beta}(\DDD)$ we have that its empirical value over a sample $S\sim\DDD^n$ is $\alpha$-close to its expected value w.p.\ $1-\beta$. Now consider a differentially private algorithm $\AAA:X^n\rightarrow \FFF_{\alpha,\beta}(\DDD)$ that takes a database and returns a function from $\FFF_{\alpha,\beta}(\DDD)$.
What can we say about the difference $|f(S)-f(\DDD^n)|$ when $f$ is chosen by $\AAA(S)$ based on the sample $S$ itself?

Using the notion of {\em max-information}, Dwork et al.~\cite{DworkFHPRR-nips-2015} showed that if $\beta$ is small enough, then w.h.p.\ the difference remains small. Informally, they showed that if $\AAA$ is differentially private, then
$$
\Pr_{\substack{S\sim\DDD^n \\ f\leftarrow\AAA(S)}}[|f(S)-f(\DDD^n)|>\alpha]\leq\beta\cdot e^{\eps^2\cdot n}.
$$
So, if $\AAA$ is a differentially private algorithm that ranges over functions which are very concentrated around their expected value (i.e., $\beta<e^{-\eps^2 n}$), then $|f(S)-f(\DDD^n)|$ remains small (w.h.p.)\ even when $f$ is chosen by $\AAA(S)$ based on the sample $S$.  When $\beta>e^{-\eps^2 n}$ it is easy to construct examples where a differentially private algorithm identifies a function $f\in\FFF_{\alpha,\beta}(\DDD)$ such that $|f(S)-f(\DDD^n)|$ is arbitrarily large with high probability. So, in general, differential privacy {\em does not} guarantee generalization for adaptively chosen functions of this sort. However, a stronger notion than differential privacy -- typical stability -- presented by Bassily and Freund~\cite{BassilyF16} does guarantee generalization in this setting. Informally, they showed that if a typically stable algorithm $\BBB$ outputs a function $f\in\FFF_{\alpha,\beta}(\DDD)$, then $|f(S)-f(\DDD^n)|$ remains small.\footnote{A similar notion -- perfect generalization -- was presented in ~\cite{CummingsLNRW16}.} 

The results of this article provide another piece of this puzzle, as we show that (a variant of) differential privacy can in some cases be used to prove that a function $f$ is in $\FFF_{\alpha,\beta}(\DDD)$.

\subsection{Our Results}

\paragraph{Notation.}
Throughout this article we use the convention that $f(\DDD^n)$ is the expected value of the function $f$ over a sample containing $n$ i.i.d.\ elements drawn according to the distribution $\DDD$. That is, $f(\DDD^n)\triangleq\ex{S\sim\DDD^n}{f(S)}$.

Fix a function $f:X^n\rightarrow\R$, let $\DDD$ be a distribution over $X$, and let $S\sim\DDD^n$.
Our goal is to bound the probability that $|f(S)-f(\DDD^n)|$ is large by some (hopefully) easy-to-analyze quantity. To intuit our result, consider for example what we get by a simple application of Markov's Inequality:
\begin{equation} 
\Pr_{S\sim\DDD^n}[|f(S)-f(\DDD^n)|>\lambda]\leq \frac{1}{\lambda}\cdot \ex{S\sim\DDD^n}{ \1_{|f(S)-f(\DDD^n)|>\lambda} \cdot |f(S)-f(\DDD^n)| }.
\label{eq:markov}
\end{equation}

We show that using differential privacy we can replace the term $|f(S)-f(\DDD^n)|$ in the expectation with $|f(S\cup\{x\})-f(S\cup\{y\})|$, which can sometimes be easier to analyze. Specifically, we show the following.

\begin{theorem}[part 1] \label{thm:dpGeneralizationIntro}
Let $\DDD$ be a distribution over a domain $X$, let $\query:X^n\rightarrow\R$ , and let $\Delta,\lambda\in\R^{\geq0}$
be s.t.\ for every $1\leq i\leq n$ it holds that
\begin{eqnarray}
\ex{\substack{S\sim\DDD^{n}\\z\sim\DDD}}{\1_{\left|\query(S) - \query\left(S^{(i\leftarrow z)}\right)\right|>\lambda} \cdot\left|\query(S) - \query\left(S^{(i\leftarrow z)}\right)\right| }\leq\Delta,
\label{eq:dpGeneralizationIntro_part1}
\end{eqnarray}
where $S^{(i\leftarrow z)}$ is the same as $S$ except that the $i^{\text{th}}$ element is replaced with $z$.
Then for every $\eps>0$ we have that
$$
\Pr_{S\sim\DDD^n}\left[ | f(S) - f(\DDD^n) | \geq 18\eps\lambda n \right] < \frac{14\Delta}{\eps\lambda},
$$
provided that $n\geq O\left(\frac{1}{\eps\cdot\min\{1,\eps\}}\log(\frac{   \lambda \cdot \min\{1,\eps\}   }{\Delta})\right)$.
\end{theorem}

Observe that for a $\lambda$-sensitive function $f$, we have that the expectation in Equation~(\ref{eq:dpGeneralizationIntro_part1}) is zero, so the statement holds for every choice of $\beta>0$ and $n\geq O\left(\frac{1}{\eps^2}\log(\frac{1}{\beta})\right)$, resulting in McDiarmid's Inequality (Equation~(\ref{eq:SU17})).
Intuitively, Theorem~\ref{thm:dpGeneralizationIntro} states that in order to obtain a high probability bound on $| f(S) - f(\DDD^n) |$ is suffices to analyze the ``expectation of the tail'' of $\left|f(S) - f\left(S^{(i\leftarrow z)}\right)\right|$, as a function of the starting point $\lambda$.

We also show that the above bound can be improved whenever the ``expectation of the head'' of $\left|f(S) - f\left(S^{(i\leftarrow z)}\right)\right|$ is smaller than $\lambda$. Specifically,

{
\renewcommand{\thetheorem}{\ref{thm:dpGeneralizationIntro}}
\begin{theorem}[part 2]
If, in addition to~(\ref{eq:dpGeneralizationIntro_part1}), $\exists \tau\leq\lambda$ s.t.\ for every $S\in X^n$ and every $1\leq i\leq n$ we have
\begin{eqnarray}
\ex{\substack{y,z\sim\DDD}}{\1_{\left|\query(S^{(i\leftarrow y)}) - \query\left(S^{(i\leftarrow z)}\right)\right|\leq\lambda} \cdot\left|\query(S^{(i\leftarrow y)}) - \query\left(S^{(i\leftarrow z)}\right)\right| }\leq\tau,
\label{eq:dpGeneralizationIntro_part2}
\end{eqnarray}
Then for every $\eps>0$ we have that
$$
\Pr_{S\sim\DDD^n}\left[ | f(S) - f(\DDD^n) | \geq 18\eps\tau n \right] < \frac{14\Delta}{\eps\tau},
$$
provided that $n\geq O\left(\frac{\lambda}{\eps\cdot\min\{1,\eps\} \tau}\log(\frac{\tau  \cdot \min\{1,\eps\} }{\Delta})\right)$
\end{theorem}
\addtocounter{theorem}{-1}
}

Observe that while the expectation in~(\ref{eq:dpGeneralizationIntro_part1}) is over the entire sample $S$ (as well as the replacement point), in requirement~(\ref{eq:dpGeneralizationIntro_part2}) the sample $S$ is fixed. We do not know if this ``worst-case'' restriction is necessary.

\medskip

In Section~\ref{sec:applications} we demonstrate how Theorem~\ref{thm:dpGeneralizationIntro} can be used in proving a variety of concentration bounds, such as a high probability bound on $|f(S)-f(\DDD^n)|$ for Lipschitz functions.
In addition we show that Theorem~\ref{thm:dpGeneralizationIntro} can be used to bound the probability that the number of triangles in a random graph significantly exceeds the expectation.  

\section{Preliminaries}

\subsection{Differential Privacy}

Our results rely on a number of basic facts about differential privacy. An algorithm operating on databases is said to preserve differential privacy if a change of a single record of the database does not significantly change the output distribution of the algorithm. Formally:

\begin{definition}
Databases $S\in X^n$ and $S'\in X^n$ over a domain $X$ are called \emph{neighboring} if they differ in exactly one entry.
\end{definition}

\begin{definition}[Differential Privacy~\cite{DMNS06,DKMMN06}]
A randomized algorithm $\AAA : X^n\rightarrow Y$ is {\em $(\epsilon,\delta)$-differentially private} if for all neighboring databases $S,S'\in X^n$, and for every set of outputs $T\subseteq Y$, we have
$$\Pr[\AAA(S)\in T]\leq e^{\eps}\cdot \Pr[\AAA(S')\in T]+\delta.$$
The probability is taken over the random coins of $\AAA$. 
\end{definition}

\subsection{The Exponential Mechanism}

We next describe the exponential mechanism of McSherry and Talwar~\cite{McSherryT07}.

\begin{definition}[Sensitivity]
The \emph{sensitivity} (or {\em global sensitivity}) of a function $f:X^n \rightarrow \R$ is the smallest $\lambda$ such that for every neighboring $S,S'\in X^n$, we have  $|f(S)-f(S')|\leq \lambda$.
We use the term ``$\lambda$-sensitive function'' to mean a function of sensitivity $\le \lambda$.
\end{definition}

Let $X$ be a domain and $H$ a set of solutions.
Given a database $S\in X^*$, the exponential mechanism privately chooses a ``good'' solution $h$ out of the possible set of solutions $H$. This ``goodness'' is quantified using a \emph{quality function} that matches solutions to scores.

\begin{definition}[Quality function]
A \emph{quality function} is a function $q:X^*\times H \rightarrow\R$ that maps a database $S\in X^*$ and a solution $h\in H$ to a real number, identified as the score of the solution $h$ w.r.t.\ the database $S$.
\end{definition}

Given a quality function $q$ and a database $S$, the goal is to chooses a solution $h$ approximately maximizing $q(S,h)$.
The exponential mechanism chooses a solution probabilistically, where the probability mass that is assigned to each solution $h$ increases exponentially with its quality $q(S,h)$:

\begin{center}
\noindent\fbox{
\parbox{.97\columnwidth}{
The Exponential Mechanism\\ 
{\bf Input:} privacy parameter $\eps>0$, finite solution set $H$, database $S\in X^n$, and a 
$\lambda$-sensitive 
quality function $q$.
\begin{enumerate}
	\item Randomly choose $h \in H$ with probability
	$\frac{\exp\left(\frac{\eps}{2\lambda} \cdot q(S,h) \right)}{\sum_{h'\in H}\exp\left(\frac{\eps}{2\lambda} \cdot q(S,h') \right)}.$
	\item Output $h$.
\end{enumerate}
}}
\end{center}

\begin{theorem}[Properties of the exponential mechanism]\label{prop:expMech}
(i) The 
exponential mechanism
 is $(\eps,0)$-differentially private. (ii)
Let $Opt(S)\triangleq\max_{f\in H}\{q(S,f)\}$ and $\Delta>0$. The 
exponential mechanism
 outputs a solution $h$ such that $q(S,h)\leq(Opt(S) - \Delta)$ with probability at most $|H| \cdot \exp\left(-\frac{\eps \Delta}{ 2 \lambda}\right)$.
\end{theorem}

\subsection{Concentration Bounds}

Let $X_1,\dots,X_n$ be independent random variables where $\Pr[X_i=1]=p$ and $\Pr[X_i=0]=1-p$ for some $0<p<1$. Clearly, $\E[\sum_{i=1}^n{X_i}]=pn$. Chernoff and Hoeffding bounds show that the sum is concentrated around this expected value:
\begin{align*}
&\Pr\left[\sum_{i=1}^n{X_i}>(1+\delta)pn\right]\leq \exp\left(-pn\delta^2/3\right) \;\;\text{ for } 0<\delta\leq 1,\\
&\Pr\left[\sum_{i=1}^n{X_i}<(1-\delta)pn\right]\leq \exp\left(-pn\delta^2/2\right) \;\;\text{ for } 0<\delta<1,\\
&\Pr\left[\left|\sum_{i=1}^n{X_i}-pn\right|>\delta\right]\leq 2\exp\left(-2\delta^2/n\right) \;\,\;\;\text{ for } \delta\geq0.
\end{align*}
The first two inequalities are known as the multiplicative Chernoff bounds~\cite{chern}, and the last inequality is known as the Hoeffding bound~\cite{hoeff}.
The next theorem states that the Chernoff bound above is tight up to constant factors in the exponent.

\begin{theorem}[Tightness of Chernoff bound~\cite{KleinY15}]\label{thm:chernoffTight}
Let $0<p,\delta\leq\frac{1}{2}$, and let $n\geq\frac{3}{\delta^2 p}$.
Let $X_1,\dots,X_n$ be independent random variables where $\Pr[X_i=1]=p$ and $\Pr[X_i=0]=1-p$. Then,
\begin{align*}
&\Pr\left[\sum_{i=1}^n{X_i}\leq(1-\delta)pn\right]\geq\exp(-9\delta^2pn),\\
&\Pr\left[\sum_{i=1}^n{X_i}\geq(1+\delta)pn\right]\geq\exp(-9\delta^2pn).
\end{align*}
\end{theorem}

\section{Concentration Bounds via Differential Privacy}

In this section we show how the concept of differential privacy can be used to derive conditions under which a function $f$ and a distribution $\DDD$ satisfy that $|f(S)-f(\DDD^n)|$ is small w.h.p.\ when $S\sim\DDD^n$.
Our proof technique builds on the proof of Bassily et al.~\cite{BassilyNSSSU16} for the generalization properties of a differentially private algorithm that outputs a low-sensitivity function. The proof consists of two steps:
\begin{enumerate}
	\item Let $S_1,\dots,S_T$ be $T$ independent samples from $\DDD^n$ (each containing $n$ i.i.d.\ samples from $\DDD$). Let $\AAA$ be selection procedure that, given $S_1,\dots,S_T$, chooses an index $t\in[T]$ with the goal of maximizing $|f(S_t)-f(\DDD^n)|$. We show that if $\AAA$ satisfies (a variant of) differential privacy then, under some conditions on the function $f$ and the distribution $\DDD$, the expectation of $|f(S_t)-f(\DDD^n)|$ is bounded. That is, if $\AAA$ is differentially private, then its ability to identify a ``bad'' index $t$ with large $|f(S_t)-f(\DDD^n)|$ is limited.
	\item We show that if $|f(S)-f(\DDD^n)|$ is large w.h.p.\ over $S\sim\DDD^n$, then it is possible to construct an algorithm $\AAA$ satisfying (a variant of) differential privacy that contradicts our expectation bound.
\end{enumerate}

We begin with a few definitions.

\subsection{Definitions}

\paragraph{Notations.}
We use $\vec{S}\in (X^n)^T$ to denote a {\em multi}-database consisting of $T$ databases of size $n$ over $X$. Given a distribution $\DDD$ over a domain $X$ we write $\vec{S}\sim\DDD^{nT}$ to denote a multi-database sampled i.i.d.\ from $\DDD$.

\begin{definition}
Fix a function $f:X^n\rightarrow\R$ mapping databases of size $n$ over a domain $X$ to the reals.
We say that two multi-databases 
$\samples=(\sample_1,\dots,\sample_T)\in(X^n)^T$ and $\samples'=(\sample'_1,\dots,\sample'_T)\in(X^n)^T$
are {\em $(f,\lambda)$-neighboring} if for all $1\leq i\leq T$ we have that
$$|f(\sample_i)- f(\sample'_i)|\leq\lambda.$$
\end{definition}

\begin{definition}[$(\eps,(f,\lambda))$-differential privacy]
Let $M:(X^n)^T\rightarrow Y$ be a randomized algorithm that operates on $T$ databases of size $n$ from $X$. For a function $f:X^n\rightarrow\R$ and parameters $\eps,\lambda\geq0$, we say that $M$ is {\em $(\eps,(f,\lambda))$-differentially private} if for every set of outputs $F\in Y$ and for every $(f,\lambda)$-neighboring $\samples,\samples'\in(X^n)^T$ it holds that
$$
\Pr[M(\samples)\in F]\leq e^\eps \cdot \Pr[M(\samples')\in F].
$$
\end{definition}

\begin{claim}\label{claim:dpExpectation}
Fix a function $f:X^n\rightarrow\R$ and parameters $\eps\leq1$ and $\lambda\geq0$.
If $M:(X^n)^T\rightarrow Y$ is $(\eps,(f,\lambda))$-differentially private then for every $(f,\lambda)$-neighboring databases $\samples,\samples'\in(X^n)^T$ and every function $h:Y\rightarrow\R$ we have that
$$
\ex{y\leftarrow M(\vec{S})}{h(y)} \leq \ex{y\leftarrow M(\vec{S'})}{h(y)} \;\;+\;\; 4\eps\cdot \ex{y\leftarrow M(\vec{S'})}{|h(y)|}.
$$
\end{claim}

Claim~\ref{claim:dpExpectation} follows from basic arguments in differential privacy. The proof appears in the appendix for completeness.

\subsection{Multi Sample Expectation Bound}

The proof of Theorem~\ref{thm:dpGeneralizationIntro} contains somewhat unwieldy notation. For readability, we present here a restricted version of the theorem, tailored to the case where the function $f$ computes the sample sum, which highlights most of the ideas in the proof. The full proof of Theorem~\ref{thm:dpGeneralizationIntro} is included in the appendix.

\paragraph{Notation.} Given a sample $S\in X^n$, we use $\bar f(S)$ to denote the sample sum, i.e., $\bar f(S)=\sum_{x\in S}x$.

\begin{lem}[Simplified Expectation Bound] \label{lem:simpleExpectationBound}
Let $\DDD$ be a distribution over a domain $X$ such that $\ex{x\sim\DDD}{x}=0$ and 
$\ex{x\sim\DDD}{\1_{\left\{|x|>1\right\}} \cdot|x| }\leq\Delta$.
Fix $0<\eps\leq1$, and let $\AAA \from (X^{n})^{\trials} \to [\trials]$ be an $(\eps,(\bar f,1))$-differentially private algorithm that operates on $T$ databases of size $n$ from $X$, and outputs an index $1\leq t \leq T$. Then 
$$
\left|\ex{\substack{\samples\sim\dist^{nT} \\ \trial \leftarrow \AAA(\samples)}}{ \bar f(S_t)  } \right| \leq 4\eps n  + 2 n T \Delta.
$$
\end{lem}

\begin{proof}
We denote $\vec{S}=(S_1,\dots,S_T)$, where every $S_t$ is itself a vector $S_t=(x_{t,1},\dots,x_{t,n})$. We have:
\begin{align}
\ex{\substack{\samples\sim\dist^{nT} \\ \trial \leftarrow \AAA(\samples)}}{ \bar f(S_t)  } 
& = \sum_{i\in[n]}\exx{ \samples\sim\DDD^{nT}}\ex{t\leftarrow\AAA(\samples)}{x_{t,i}}  \nonumber\\
&= \sum_{i\in[n]}\ex{\samples\sim\DDD^{nT}}{\1\left\{ \max_{m\in[t]}|x_{m,i}|\leq1  \right\}\cdot\ex{t\leftarrow\AAA(\samples)}{x_{t,i}}
+ \1\left\{ \max_{m\in[t]}|x_{m,i}|>1  \right\}\cdot\ex{t\leftarrow\AAA(\samples)}{x_{t,i}}}. \quad \label{eq:warmup1}
\end{align}

In the case where $\max_{m\in[t]}|x_{m,i}|>1$ we replace the expectation over $t\leftarrow\AAA(\vec{S})$ with the deterministic choice for the maximal $t$ (this makes the expression larger).
When $\max_{m\in[t]}|x_{m,i}|\leq1$ we can use the privacy guarantees of algorithm $\AAA$.
Given a multi-sample $\vec{S}\in(X^n)^T$ we use $\vec{S}_{-i}$ to denote a multi-sample identical to $\vec{S}$, except that the $i^{\text{th}}$ element of {\em every} sub-sample is replaced with 0.
Using Claim~\ref{claim:dpExpectation} we get

\begin{align}
(\ref{eq:warmup1}) \; &\leq 
\sum_{i\in[n]}\ex{\samples\sim\DDD^{nT}}{\1\left\{ \max_{m\in[t]}|x_{m,i}|\leq1  \right\}\cdot\left(\ex{t\leftarrow\AAA(\samples_{-i})}{x_{t,i}} 
+4\eps \ex{t\leftarrow\AAA(\samples_{-i})}{|x_{t,i}|} \right)
+ \1\left\{ \max_{m\in[t]}|x_{m,i}|>1  \right\}\cdot\max_{m\in[T]}|x_{m,i}|}\nonumber\\
&\leq 4\eps n\;+\;
\sum_{i\in[n]}\ex{\samples\sim\DDD^{nT}}{\1\left\{ \max_{m\in[t]}|x_{m,i}|\leq1  \right\}\cdot\ex{t\leftarrow\AAA(\samples_{-i})}{x_{t,i}} 
+ \1\left\{ \max_{m\in[t]}|x_{m,i}|>1  \right\}\cdot\max_{m\in[T]}|x_{m,i}|} \label{eq:warmup2}
\end{align}

We next want to remove the first indicator function. This is useful as without it, the expectation of a fresh example from $\DDD$ is zero. To that end we add and subtract the expression $\1\left\{ \max_{m\in[t]}|x_{m,i}|>1  \right\}\cdot\ex{t\leftarrow\AAA(\samples_{-i})}{x_{t,i}}$ to get (after replacing again $\E_t$ with $\max_t$)

\begin{align*}
(\ref{eq:warmup2})\; &\leq 4\eps n\;+\;
\sum_{i\in[n]}\ex{\samples\sim\DDD^{nT}}{\ex{t\leftarrow\AAA(\samples_{-i})}{x_{t,i}} 
\;+\; 2\cdot \1\left\{ \max_{m\in[t]}|x_{m,i}|>1  \right\}\cdot\max_{m\in[T]}|x_{m,i}|}\\
&\leq 4\eps n\;+\;2 \sum_{i\in[n]} \sum_{m\in[T]} \ex{\samples\sim\DDD^{nT}}{ 
  \1\left\{ |x_{m,i}|>1  \right\}\cdot|x_{m,i}|}\\
&\leq 4\eps n\;+\;2nT\Delta.
\end{align*}
\end{proof}

\subsection{Multi Sample Amplification}

\begin{theorem}[Simplified High Probability Bound] \label{thm:simplifiedDpGeneralization}
Let $\DDD$ be a distribution over a domain $X$ such that $\ex{x\sim\DDD}{x}=0$. Let $\Delta\geq0$ be such that
$\ex{x\sim\DDD}{\1_{\left\{|x|>1\right\}} \cdot|x| }\leq\Delta$.
Fix $1\geq \eps \geq \sqrt{\frac{1}{n}\ln(2/\Delta)}$. We have that
$$
\Pr_{S\sim\DDD^n}\left[ | \bar f(S) | \geq 30\eps n \right] < \frac{\Delta}{\eps}.
$$
\end{theorem}

We present the proof idea of the theorem. Any informalities made hereafter are removed in Section~\ref{sec:fullProof}.

\begin{proof}[Proof sketch]
We only analyze the probability that $\bar f(S)$ is large. The analysis is symmetric for when $\bar f(S)$ is small.
Assume towards contradiction that with probability at least $\frac{\Delta}{2\eps}$ we have that $ \bar f(S)  \geq 30\eps n$. We now construct the following algorithm $\BBB$ that contradicts our expectation bound.

\begin{algorithm}[H]
\caption{$\BBB$}\addcontentsline{lof}{figure}{Algorithm $\BBB$}
\vspace{2pt}
{\bf Input:} $T$ databases of size $n$ each: $\vec{S}=(S_1,\dots,S_T)$, where $T\triangleq\left\lfloor 2\eps/\Delta \right\rfloor$.
\begin{enumerate}[rightmargin=10pt,itemsep=1pt,topsep=4pt]

\item For $i\in[T]$, define $q(\vec{S},i) =  \bar f(S_i) $. 

\item Sample $t^*\in [T]$ with probability proportional to $\exp\left(\frac{\eps}{2} q(\vec{S},t)\right)$.

\end{enumerate}
\textbf{Output:} $t.$
\end{algorithm}

The fact that algorithm $\BBB$ is $(\eps,(\bar f,1))$-differentially private follows from the standard analysis of the Exponential Mechanism of McSherry and Talwar~\cite{McSherryT07}. The analysis appears in the full version of this proof (Section~\ref{sec:fullProof}) for completeness.

Now consider applying $\BBB$ on databases $\vec{S} = (S_1,\dots,S_T)$ containing i.i.d.\ samples from $\DDD$. By our assumption on $\DDD$, for every $t$ we have that 
$\bar f(S_t) \geq 30\eps n$
 with probability at least $\frac{\Delta}{2\eps}$. By our choice of $T = \left\lfloor 2\eps/\Delta \right\rfloor$, we therefore get
$$\Pr_{\vec{S}\sim\DDD^{nT}}\left[{\max_{t \in [T]}  \left\{   \bar f(S_t)  \right\} \geq 30\eps n }\right] \geq 1 - \left( 1 - \frac{\Delta}{2\eps} \right)^T \geq \frac12.$$
The probability is taken over the random choice of
the examples in $\vec{S}$ according to $\DDD$.
Had it been the case that the random variable $\max_{t \in [T]}  \left\{   \bar f(S_t)  \right\}$ is non-negative, we could have used Markov's inequality to get
\begin{equation}\label{eq:warmupLargeError}
\E_{\vec{S}\sim\DDD^{nT}}\left[\max_{t \in [T]} \left\{ q(\vec{S},t) \right\}\right] =
\E_{\vec{S}\sim\DDD^{nT}}\left[\max_{t \in [T]}  \left\{   \bar f(S_t)  \right\}\right] \geq 15\eps n.
\end{equation}

Even though it is not the case that $\max_{t \in [T]}  \left\{   \bar f(S_t)  \right\}$  is non-negative, we now proceed as if Equation~(\ref{eq:warmupLargeError}) holds.
As described in the full version of this proof (Section~\ref{sec:fullProof}), this technical issue has an easy fix.
So, in expectation, $\max_{t \in [T]}  \left(q(\vec{S},t)\right)$ is large. In order to contradict the expectation bound of Theorem~\ref{thm:dpGeneralization}, we need to show that this is also the case for the index $t^*$ that is sampled on Step~2. To that end, we now use the following technical claim, stating that the expected quality of a solution sampled as in Step~2 is high.

\begin{claim}[e.g.,~\cite{BassilyNSSSU16}] \label{claim:EMutility}
Let $H$ be a finite set, $h : H \to \mathbb{R}$ a function, and $\eta >0$. Define a random variable $Y$ on $H$ by $\Pr[Y=y] = \exp(\eta h(y))/C$, where $C= \sum_{y \in H} \exp(\eta h(y))$. Then $\ex{}{h(Y)} \geq \max_{y \in H} h(y) - \frac{1}{\eta}\ln |H|$.
\end{claim}

For every fixture of $\vec{S}$, we can apply Claim~\ref{claim:EMutility} with $h(t) =  q(\vec{S},t)$ and $\eta = \frac{\eps}{2}$ to get 
\begin{equation*}
\E_{t^*\in_R [T]}[q(\vec{S},t^*)]
=\E_{t^*\in_R [T]}\Big[ \bar f(S_{t^*}) \Big] 
 \geq \max_{t \in [T]}  \left\{   \bar f(S_t)  \right\} - \frac{2}{\eps} \ln(T).
\end{equation*}
Taking the expectation also over $\vec{S}\sim\DDD^{nT}$ we get that
\begin{eqnarray*}
\E_{\substack{\vec{S}\sim\DDD^{nT} \\ t^*\leftarrow\BBB\left(\vec{S}\right)}}\Big[\bar f(S_{t^*})\Big] 
&\geq& \E_{\vec{S}\sim\DDD^{nT}}\left[\max_{t \in [T]}  \left\{   \bar f(S_t)  \right\}\right] - \frac{2}{\eps} \ln(T)\\
&\geq& 15\eps n - \frac{2}{\eps} \ln(T).
\end{eqnarray*}
This contradicts Theorem~\ref{thm:dpGeneralization} whenever $\eps>\sqrt{\frac{1}{n}\ln(T)}=\sqrt{\frac{1}{n}\ln(2\eps/\Delta)}$.
\end{proof}

\section{Applications}\label{sec:applications}

In this section we demonstrate how Theorem~\ref{thm:dpGeneralizationIntro} can be used in proving a variety of concentration bounds.

\subsection{Example: Subgaussian Diameter and Beyond}

Recall that for a low-sensitivity function $f$, one could use McDiarmid's Inequality to obtain a high probability bound on the difference $|f(S)-f(\DDD^n)|$, and this bound is {\em distribution-independent}. That is, the bound does not depend on $\DDD$.
Over the last few years, there has been some work on providing distribution-dependent refinements to McDiarmid's Inequality, that hold even for functions with high worst-case sensitivity, but with low ``average-case'' sensitivity, where ``average'' is with respect to the underlying distribution $\DDD$.
The following is one such refinement, by Kontorovich~\cite{Kontorovich14}.

\begin{definition}[\cite{Kontorovich14}]
Let $\DDD$ be a distribution over a domain $X$, and let $\rho:X^2\rightarrow\R^{\geq0}$. The {\em symmetrized distance} of $(X,\rho,\DDD)$ is the random variable $\Xi= \xi\cdot \rho(x,x')$ where $x,x'\sim\DDD$ are independent and $\xi$ is uniform on $\{\pm1\}$ independent of $x,x'$.
The {\em subgaussian diameter} of $(X,\rho,\DDD)$, denoted $\Delta_{\rm{SG}}(X,\rho,\DDD)$, is the smallest $\sigma\in\R^{\geq0}$ such that
$$
\E\left[  e^{\lambda \Xi} \right] \leq e^{\sigma^2\lambda^2/2}, \;\;\; \forall\lambda\in\R.
$$
\end{definition}

In~\cite{Kontorovich14}, Kontorovich showed the following theorem:

\begin{theorem}[\cite{Kontorovich14}, informal]\label{thm:Kontorovich}
Let $f:X^n\rightarrow\R$ be a function mapping databases of size $n$ over a domain $X$ to the reals. 
Assume that there exists a function $\rho:X^2\rightarrow\R^{\geq0}$ s.t.\ for every $i\in[n]$, every $S\in X^n$, and every $y,z\in X$ we have that
$$
\left|f\left(S^{(i\leftarrow y)}\right) - f\left(S^{(i\leftarrow z)}\right)\right|\leq\rho(y,z),
$$
where $S^{(i\leftarrow x)}$ is the same as $S$ except that the $i^{\text{th}}$ element is replaced with $x$. Then,
$$
\Pr_{S\sim\DDD^n}[|f(S)-f(\DDD^n)|\geq t]\leq 2\exp\left(- \frac{t^2}{2 n\cdot  \Delta_{\rm{SG}}^2(X,\rho,\DDD) } \right).
$$
\end{theorem}

Informally, using the above theorem it is possible to obtain concentration bounds for functions with unbounded sensitivity (in worst case), provided that the sensitivity (as a random variable) is subgaussian. 
In this section we show that our result implies a similar version of this theorem.
While the bound we obtain is weaker then Theorem~\ref{thm:Kontorovich}, our techniques can be extended to obtain concentration bounds even in cases where the sensitivity is {\em not} subgaussian (that is, in cases where the subgaussian diameter is unbounded, and hence, Theorem~\ref{thm:Kontorovich} could not be applied).

Let us denote $\sigma=\Delta_{\rm{SG}}(X,\rho,\DDD)$. Now for $t\geq0$,

\begin{align}
\Pr_{x,y\sim\DDD}[\rho(x,y)\geq t] &\leq 2 \Pr_{\substack{x,y\in\DDD\\\xi\in\{\pm1\}}}[\xi\cdot\rho(x,y)\geq t] 
= 2 \Pr[\Xi\geq t] 
= 2 \Pr[e^{\frac{t}{\sigma^2}\cdot\Xi}\geq e^{\frac{t}{\sigma^2}\cdot t}] \nonumber\\
&\leq 2 e^{-\frac{t^2}{\sigma^2}} \cdot \E\left[ e^{\frac{t}{\sigma^2}\cdot\Xi} \right]
 \leq 2 e^{-\frac{t^2}{\sigma^2}} \cdot e^{\frac{\sigma^2}{2}\cdot \frac{t^2}{\sigma^4}} = 2\exp\left(  -\frac{t^2}{2\sigma^2} \right).
\end{align}

So,

\begin{align}
& \ex{\substack{S\sim\DDD^{n}\\x'\sim\DDD}}{\1\left\{\left|\query(S) - \query\left(S^{(i\leftarrow x')}\right)\right|>\lambda\right\} \cdot\left|\query(S) - \query\left(S^{(i\leftarrow x')}\right)\right| } \nonumber\\
&\qquad \leq\ex{x,y\sim\DDD}{\1\left\{\rho(x,y)>\lambda\right\} \cdot \rho(x,y) } \nonumber\\
&\qquad = \int_0^\lambda \Pr_{x,y\sim\DDD}\left[\1\left\{\rho(x,y)>\lambda\right\} \cdot \rho(x,y)\geq t\right] {\rm{d}}t
 \;+\; \int_\lambda^\infty \Pr_{x,y\sim\DDD}\left[\1\left\{\rho(x,y)>\lambda\right\} \cdot \rho(x,y)\geq t\right] {\rm{d}}t \nonumber\\
&\qquad = \int_0^\lambda \Pr_{x,y\sim\DDD}\left[ \rho(x,y)\geq \lambda \right] {\rm{d}}t
 \;+\; \int_\lambda^\infty \Pr_{x,y\sim\DDD}\left[ \rho(x,y)\geq t\right] {\rm{d}}t \nonumber\\
&\qquad = \lambda \cdot \Pr_{x,y\sim\DDD}\left[ \rho(x,y)\geq \lambda \right] 
 \;+\; \int_\lambda^\infty \Pr_{x,y\sim\DDD}\left[ \rho(x,y)\geq t\right] {\rm{d}}t \nonumber\\
&\qquad \leq \lambda\cdot 2\exp\left(  -\frac{\lambda^2}{2\sigma^2} \right)
\;+\; \int_\lambda^\infty 2\exp\left(  -\frac{t^2}{2\sigma^2} \right) {\rm{d}}t \nonumber\\
&\qquad = \lambda\cdot 2\exp\left(  -\frac{\lambda^2}{2\sigma^2} \right)
\;+\; \sqrt{2\pi} \sigma \cdot {\rm{erfc}}\left(  \frac{\lambda}{\sqrt{2}\sigma} \right)  \nonumber\\
&\qquad \leq \lambda\cdot 2\exp\left(  -\frac{\lambda^2}{2\sigma^2} \right)
\;+\; \sqrt{2\pi} \sigma \cdot \exp\left(-  \frac{\lambda^2}{2\sigma^2} \right)
\leq 3(\lambda+\sigma)\cdot \exp\left(-  \frac{\lambda^2}{2\sigma^2} \right)
 \triangleq \Delta. \nonumber
\end{align}

In order to apply Theorem~\ref{thm:dpGeneralizationIntro} we need to ensure that $n\geq	O\left(\frac{1}{\eps\cdot\min\{1,\eps\}}\ln\left(\frac{\lambda\cdot\min\{1,\eps\}}{\Delta}\right)\right)$.
For our choice of $\Delta$, it suffices to set $\eps_0= \Theta \left(\frac{\lambda}{\sqrt{n}\sigma}\right)$, assuming that $\frac{\lambda}{\sqrt{n}\sigma}\leq1$. Otherwise, if $\frac{\lambda}{\sqrt{n}\sigma}>1$, we will choose $\eps_1=\Theta\left(\frac{\lambda^2}{n\sigma^2}\right)$.
Plugging $(\eps_0,\Delta)$ or $(\eps_1,\Delta)$ into Theorem~\ref{thm:dpGeneralizationIntro}, and simplifying, we get
\begin{eqnarray}
\Pr_{S\sim\DDD}\left[ |f(S)-f(\DDD^n)| \geq t \right] \leq 
\left\{ \begin{array}{ccl}
	e^{-\Omega\left(\frac{t}{\sqrt{n}\sigma}\right)} & ,& t \leq \sigma\cdot n^{1.5}\\[0.5em]
	e^{-\Omega\left(\frac{t^{2/3}}{\sigma^{2/3}}\right)} & ,& t > \sigma\cdot n^{1.5} \\
\end{array} \right.
\label{eq:KontorovichSimilar}
\end{eqnarray}

Clearly, the bound of Theorem~\ref{thm:Kontorovich} is stronger. 
Note, however, that the only assumption we used here is that $\int_\lambda^\infty \Pr_{x,y\sim\DDD}[\rho(x,y)\geq t] {\rm{d}}t$ is small. Hence, as the following section shows, this argument could 
be extended to obtain concentration bounds even when $\Delta_{\rm{SG}}(X,\rho,\DDD)$ is unbounded. We remark that Inequality~\ref{eq:KontorovichSimilar} can be slightly improved by using part~2 of Theorem~\ref{thm:dpGeneralizationIntro}. This will be illustrated in the following section.

\subsection{Example: Concentration Under Infinite Variance}

Let $f:X^n\rightarrow\R$ be a function mapping databases of size $n$ over a domain $X$ to the reals. 
Assume that there exists a function $\rho:X^2\rightarrow\R^{\geq0}$ s.t.\ for every $i\in[n]$, every $S\in X^n$, and every $y,z\in X$ we have that
$$
\left|f\left(S^{(i\leftarrow y)}\right) - f\left(S^{(i\leftarrow z)}\right)\right|\leq\rho(y,z),
$$
where $S^{(i\leftarrow x)}$ is the same as $S$ except that the $i^{\text{th}}$ element is replaced with $x$.\\

As stated in the previous section, the results of~\cite{Kontorovich14} can be used to obtain a high probability bound on $|f(S)-f\left(\DDD^n\right)|$ whenever $\Pr_{x,y\sim\DDD}[\rho(x,y)\geq t]\leq\exp\left( -t^2/\sigma^2 \right)$ for some $\sigma>0$. In contrast, our bound can be used whenever $\int_\lambda^\infty \Pr_{x,y\sim\DDD}[\rho(x,y)\geq t] {\rm{d}}t$ is finite.
In particular, we now use it
to obtain a concentration bound for a case where the probability distribution
of $\rho(x,y)$ is heavy tailed, and in fact, has infinite variance. Specifically, assume that
all we know on $\rho(x,y)$ is that $\Pr[\rho(x,y)\geq t]\leq 1/t^2$ for every $t\geq 1$ (this is a special case of the {\em Pareto distribution}, with infinite variance).
Let $\lambda\geq1$. We calculate:

\begin{align*}
& \ex{\substack{S\sim\DDD^{n}\\x'\sim\DDD}}{\1\left\{\left|\query(S) - \query\left(S^{(i\leftarrow x')}\right)\right|>\lambda\right\} \cdot\left|\query(S) - \query\left(S^{(i\leftarrow x')}\right)\right| } \\
&\qquad \leq\ex{x,y\sim\DDD}{\1\left\{\rho(x,y)>\lambda\right\} \cdot \rho(x,y) } \\
&\qquad = \int_0^\lambda \Pr_{x,y\sim\DDD}\left[\1\left\{\rho(x,y)>\lambda\right\} \cdot \rho(x,y)\geq t\right] {\rm{d}}t
 \;+\; \int_\lambda^\infty \Pr_{x,y\sim\DDD}\left[\1\left\{\rho(x,y)>\lambda\right\} \cdot \rho(x,y)\geq t\right] {\rm{d}}t \\
&\qquad = \int_0^\lambda \Pr_{x,y\sim\DDD}\left[ \rho(x,y)\geq \lambda \right] {\rm{d}}t
 \;+\; \int_\lambda^\infty \Pr_{x,y\sim\DDD}\left[ \rho(x,y)\geq t\right] {\rm{d}}t \\
&\qquad = \lambda \cdot \Pr_{x,y\sim\DDD}\left[ \rho(x,y)\geq \lambda \right] 
 \;+\; \int_\lambda^\infty \Pr_{x,y\sim\DDD}\left[ \rho(x,y)\geq t\right] {\rm{d}}t \\
&\qquad \leq \lambda\frac{1}{\lambda^2}
\;+\; \int_\lambda^\infty \frac{1}{t^2} {\rm{d}}t 
=\frac{2}{\lambda} \triangleq \Delta.
\end{align*}

In order to apply Theorem~\ref{thm:dpGeneralizationIntro} we need to ensure that 
$n\geq O\left( \frac{1}{\eps\cdot\min\{1,\eps\}}\ln\left(\frac{\lambda\cdot\min\{1,\eps\}}{\Delta}+1\right)\right)$.
Assuming that $n\geq \ln(\lambda)$, with our choice of $\Delta$ it suffices to set $\eps=\Theta\left(\sqrt{\frac{1}{n}\ln(\lambda)}\right)$.
Plugging $\eps$ and $\Delta$ into Theorem~\ref{thm:dpGeneralizationIntro}, and simplifying, we get
\begin{eqnarray}
\Pr_{S\sim\DDD}\left[ |f(S)-f(\DDD^n)| \geq t \right] \leq \tilde{O}\left(\frac{n^{3/2}}{t^2}\right).
\label{eq:ParetoTailBound}
\end{eqnarray}

Observe that the above bound decays as $1/t^2$. This should be contrasted with Markov's Inequality, which would decay as $1/t$. 
Recall the assumption that the variance of $\rho(x,y)$ is unbounded. Hence, the variance of $f(S)$ can also be unbounded, and Chebyshev's inequality could not be applied.\\

As we now explain, Inequality~\ref{eq:ParetoTailBound} can be improved using part~2 of Theorem~\ref{thm:dpGeneralizationIntro}. To that end, for a fixed database $S\in X^n$, we calculate:

\begin{align*}
&\ex{y,z\sim\DDD}{\1\left\{\left|f(S^{(i\leftarrow y)}) - f\left(S^{(i\leftarrow z)}\right)\right|\leq\lambda\right\} \cdot\left|f(S^{(i\leftarrow y)}) - f\left(S^{(i\leftarrow z)}\right)\right| }\\
&\leq{} \ex{y,z\sim\DDD}{\rho(y,z) }
\leq{} \int_0^1 1{\rm{d}}t + \int_1^\infty \frac{1}{t^2} {\rm{d}}t
=2\triangleq\tau.
\end{align*}

In order to apply part~2 of Theorem~\ref{thm:dpGeneralizationIntro} we need to ensure that 
$n\geq O\left( \frac{\lambda}{\eps\cdot\min\{1,\eps\}\tau}\ln\left(\frac{\eps\tau}{\Delta}\right)\right)$.
For our choice of $\Delta$ and $\tau$, if $n\geq\lambda\ln(\lambda)$ then it suffices to set $\eps_0=\Theta\left( \sqrt{\frac{\lambda}{n}\ln(\lambda) } \right)$. Otherwise, if $n<\lambda\ln(\lambda)$ then it suffices to set $\eps_1=\Theta\left( \frac{\lambda}{n}\ln(\lambda)  \right)$.
Plugging $(\eps_0,\Delta)$ or $(\eps_1,\Delta)$ into Theorem~\ref{thm:dpGeneralizationIntro}, and simplifying, we get
\begin{eqnarray*}
\Pr_{S\sim\DDD}\left[ |f(S)-f(\DDD^n)| \geq t \right] \leq 
\left\{ \begin{array}{ccl}
	\tilde{O}\left(\frac{n^{2}}{t^3}\right) & ,& t \leq n\\[0.5em]
	\tilde{O}\left(\frac{n}{t^2}\right) & ,& t > n \\
\end{array} \right.
\end{eqnarray*}

\subsection{Example: Triangles in Random Graphs}

A random graph $G(N,p)$ on $N$ vertices $1,2,\dots,N$ is defined by drawing an edge between each pair $1\leq i<j\leq N$ independently with probability $p$. There are $n={{N}\choose{2}}$ i.i.d.\ random variables $x_{\{i,j\}}$ representing the choices: $x_{\{i,j\}}=x_{\{j,i\}}=1$ if the edge $\{i,j\}$ is drawn, and 0 otherwise. We will use $\DDD$ to denote the probability $\Pr_{x\sim\DDD}[x=1]=p$ and $\Pr_{x\sim\DDD}[x=0]=1-p$, and let $S=\left(x_{\{1,2\}},\dots,x_{\{n-1,n\}}\right)\sim\DDD^n$.

We say that three vertices $i,j,\ell$ form a triangle if there is an edge between any pair of them. Denote $f_{K_3}(S)$ the number of triangles in the graph defined by $S$. For a small constant $\alpha$, we would like to have an exponential bound on the following probability
$$
\Pr\left[ f_{K_3}(S) \geq (1+\alpha) \cdot f_{K_3}(\DDD^n) \right].
$$


Specifically, we are interested in small values of $p=o(1)$ such that $f_{K_3}(\DDD^n)={{N}\choose{3}} p^3 = \Theta\left(N^3 p^3\right)=o(N)$. The difficulty with this choice of $p$ is that (in worst-case) adding a single edge to the graph can increase the number of triangles by $(N-2)$, which is much larger then the expected number of triangles. Indeed, until the breakthrough work of Vu~\cite{Vu2002} in 2002, no general exponential bounds were known. Following the work of~\cite{Vu2002}, in 2004 Kim and Vu~\cite{KimVu2004} presented the following sharp bound:

\begin{theorem}[\cite{KimVu2004}, informal]\label{thm:KimVu}
Let $\alpha$ be a small constant. It holds that
$$
\exp\left( -\Theta\left( p^2 N^2 \log(1/p) \right) \right)
\leq\Pr_{S\sim\DDD^n}\left[ f_{K_3}(S) \geq (1+\alpha)\cdot f_{K_3}(\DDD^n) \right]
\leq\exp\left( -\Theta\left( p^2 N^2 \right) \right).
$$
\end{theorem}

In this section we show that our result can be used to analyze this problem. While the bound we obtain is much weaker than Theorem~\ref{thm:KimVu}, we find it interesting that the same technique from the last sections can also be applied here. To make things more concrete, we fix
$$
p=N^{-3/4}.
$$

In order to use our concentration bound, we start by analyzing the expected difference incurred to $f_{K_3}$ by resampling a single edge. We will denote $\blacktriangle_{i,j}(S)$ as the number of triangles that are created (or deleted) by adding (or removing) the edge $\{i,j\}$. That is,
$$
\blacktriangle_{i,j}(S) = \left|\left\{ \ell\neq i,j \;:\; x_{\{i,\ell\}}=1 \text{ and } x_{\{\ell,j\}}=1 \right\}\right|.
$$
Observe that $\blacktriangle_{i,j}(S)$ does not depend on $x_{\{i,j\}}$. 
Moreover, observe that for every fixture of $i<j$ we have that $\blacktriangle_{i,j}(S)$ is the sum of $(N-2)$ i.i.d.\ indicators, each equals to 1 with probability $p^2$.

Fix $S=\left(x_{\{1,2\}},\dots,x_{\{n-1,n\}}\right)\in\{0,1\}^n$ and $x'\in\{0,1\}$. We have that
$$
\left|f_{K_3}(S) - f_{K_3}\left(S^{(\{i,j\}\leftarrow x')}\right)\right| = 
\left\{ \begin{array}{ccl}
	0 & ,& x_{\{i,j\}}=x'\\
	\blacktriangle_{i,j}(S) & ,& x_{\{i,j\}} \neq x' \\
\end{array} \right.
$$
where $S^{(\{i,j\}\leftarrow x')}$ is the same as $S$ except with $x_{\{i,j\}}$ replaced with $x'$. Fix $i<j$. 
We can now calculate

\begin{align}
&\ex{\substack{S\sim\DDD^{n}\\x'\sim\DDD}}{\1\left\{\left|f_{K_3}(S) - f_{K_3}\left(S^{(\{i,j\}\leftarrow x')}\right)\right|>\lambda\right\} \cdot\left|f_{K_3}(S) - f_{K_3}\left(S^{(\{i,j\}\leftarrow x')}\right)\right| } \nonumber\\
&={} \ex{\substack{S\sim\DDD^{n}\\x'\sim\DDD}}{\1\left\{x_{\{i,j\}\neq x'}\right\}\cdot\1\left\{\blacktriangle_{i,j}(S)>\lambda\right\} \cdot \blacktriangle_{i,j}(S) } \nonumber\\
&={} \Pr_{x_{\{i,j\}},x'\sim\DDD}\left[x_{\{i,j\}}\neq x'\right] \cdot \ex{S\sim\DDD^{n}}{\1\left\{\blacktriangle_{i,j}(S)>\lambda\right\} \cdot \blacktriangle_{i,j}(S) } \nonumber\\
&={} 2p(1-p)\cdot \left( \lambda\cdot\Pr_{S\sim\DDD^{n}}[\blacktriangle_{i,j}(S) \geq \lambda] +
\int_{\lambda}^N \Pr_{S\sim\DDD^{n}}[\blacktriangle_{i,j}(S) \geq t] {\rm{d}}t \right) \nonumber\\
%
%
&\leq{} 2pN \cdot \Pr_{S\sim\DDD^{n}}[\blacktriangle_{i,j}(S) \geq \lambda].
\label{eq:app13}
\end{align}

Recall that $\blacktriangle_{i,j}(S)$ is the sum of $(N-2)$ i.i.d.\ indicators, each equals to 1 with probability $p^2$.
We can upper bound the probability that $\blacktriangle_{i,j}(S)\geq\lambda$ with the probability that a sum of $N$ such random variables is at least $\lambda$. 
We will use the following variant of the Chernoff bound, known as the Chernoff-Hoeffding theorem:

\begin{theorem}[\cite{hoeff}]\label{thm:Chernoff_entroty}
Let $X_1,\dots,X_n$ be independent random variables where $\Pr[X_i=1]=p$ and $\Pr[X_i=0]=1-p$ for some $0<p<1$. Let $k$ be s.t.\ $p<\frac{k}{n}<1$. Then,
$$
Pr\left[\sum_{i=1}^n{X_i}\geq k\right]\geq\exp\left(- n \cdot D\left(\left.\frac{k}{n}\right\|p\right)\right),
$$
where $D(a\|b)$ is the relative entropy between an $a$-coin and a $p$-coin (i.e. between the Bernoulli($a$) and Bernoulli($p$) distribution):
$$
D(a\|p) = a\cdot \log\left(\frac{a}{p}\right) + (1-a)\cdot \log\left(\frac{1-a}{1-p}\right).
$$
\end{theorem}

Using the Chernoff-Hoeffding theorem, for $p^2 N<\lambda<N$, we have

\begin{align}
(\ref{eq:app13})&\leq{} 2pN\cdot  \exp\left(- N \cdot D\left(  \left. \frac{\lambda}{N} \right\| p^2 \right) \right).  \label{eq:app14}
\end{align}

Recall that we fixed $p=N^{-3/4}$. Choosing $\lambda=N^{1/13}$, we get:

\begin{align}
(\ref{eq:app14})&={} 2pN\cdot  \exp\left(- N \cdot D\left(  \left. N^{-12/13} \right\| N^{-6/4} \right) \right).\label{eq:app15}
\end{align}

We will use the following claim to bound $D\left(  \left. N^{-12/13} \right\| N^{-6/4} \right)$: 

\begin{claim}\label{claim:relativeEntropy}
Fix constants $c>b>0$. For $N\geq \max\{ 2^{1/b} , 2^{8/(c-b)} \}$ we have that
$D\left( \left.  N^{-b}  \right\| N^{-c} \right) \geq \frac{c-b}{2}\cdot N^{-b} \cdot \log(N)$. 
\end{claim}

Using Claim~\ref{claim:relativeEntropy}, for large enough $N$, we have that

\begin{align}
(\ref{eq:app15})
&\leq{} 2pN\cdot  \exp\left(- N^{1/13} \right). 
\label{eq:app17}
\end{align}

So, denoting $\Delta= 2pN\cdot  \exp\left(- N^{1/13} \right)$, we get that

$$
\ex{\substack{S\sim\DDD^{n}\\x'\sim\DDD}}{\1\left\{\left|f_{K_3}(S) - f_{K_3}\left(S^{(\{i,j\}\leftarrow x')}\right)\right|>\lambda\right\} \cdot\left|f_{K_3}(S) - f_{K_3}\left(S^{(\{i,j\}\leftarrow x')}\right)\right| } \leq \Delta.
$$

In order to obtain a meaningful bound, we will need to use part~2 of Theorem~\ref{thm:dpGeneralizationIntro}. 
To that end, for every fixture of $S\in X^n$ and $i<j$ we can compute

\begin{align*}
\ex{y,z\sim\DDD}{\1\left\{\left|f_{K_3}(S^{(\{i,j\}\leftarrow y)}) - f_{K_3}\left(S^{(\{i,j\}\leftarrow z)}\right)\right|\leq\lambda\right\} \cdot\left|f_{K_3}(S^{(\{i,j\}\leftarrow y)}) - f_{K_3}\left(S^{(\{i,j\}\leftarrow z)}\right)\right| }
&\leq\ex{y,z\sim\DDD}{\1\left\{ y\neq z \right\}\cdot\lambda}\\
&= 2p(1-p)\lambda \leq 2p\lambda \triangleq \tau.
\end{align*}

Finally, in order to apply Theorem~\ref{thm:dpGeneralizationIntro}, we need to ensure that $n\geq	O\left( \frac{\lambda}{\eps\min\{1,\eps\}\tau}\ln\left(\frac{\min\{1,\eps\}\tau}{\Delta}\right) \right)$. 
With our choices for $\Delta$ and $\tau$, it suffices to set
$\eps = \Theta\left( \sqrt{\frac{\lambda}{n p} } \right)$.
Plugging $\eps$, $\Delta$ and $\tau$ into Theorem~\ref{thm:dpGeneralizationIntro}, and simplifying, we get that

$$
\Pr_{S\sim\DDD^n}\left[ | f_{K_3}(S) - f_{K_3}(\DDD^n) | \geq o\left( f_{K_3}(\DDD^n) \right) \right] < \exp\left(-  N^{1/13}\right).
$$

It remains to prove Claim~\ref{claim:relativeEntropy}:

{
\renewcommand{\thetheorem}{\ref{claim:relativeEntropy}}
\begin{claim}
Fix constants $c>b>0$. For $N\geq \max\{ 2^{1/b} , 2^{8/(c-b)} \}$ we have that
$D\left( \left.  N^{-b}  \right\| N^{-c} \right) \geq \frac{c-b}{2}\cdot N^{-b} \cdot \log(N)$. 
\end{claim}
\addtocounter{theorem}{-1}
}

\begin{proof}[Proof of Claim~\ref{claim:relativeEntropy}]

\begin{align}
D\left( \left.  N^{-b}  \right\| N^{-c} \right) 
&={} N^{-b} \cdot \log\left( N^{c-b} \right) + \left( 1 - N^{-b}\right) \cdot \log\left(  \frac{1 - N^{-b}}{1 - N^{-c}}  \right) \nonumber\\
&={} N^{-b} \cdot \log\left( N^{c-b} \right) + \left( 1 - N^{-b}\right) \cdot \log\left( \frac{N^{c}-N^{c-b}}{N^{c}-1}  \right) \nonumber\\
&={} N^{-b} \cdot \log\left( N^{c-b} \right) + \left( 1 - N^{-b}\right) \cdot \log\left( 1- \frac{N^{c-b}-1}{N^{c}-1}  \right) \label{eq:appendix1}
\end{align}

Using the fact that $\log(1-x)\geq -2x$ for every $0\leq x\leq\frac{1}{2}$, and assuming that $N\geq 2^{1/b}$, we have that

\begin{align}
(\ref{eq:appendix1}){} 
& \geq {} N^{-b} \cdot \log\left( N^{c-b} \right) -2 \left( 1 - N^{-b}\right) \cdot \frac{N^{c-b}-1}{N^{c}-1} \nonumber\\
&={} N^{-b} \cdot \log\left( N^{c-b} \right) - 2 \cdot \frac{N^{c-b}-1}{N^{c}-1} + 2 N^{-b} \cdot \frac{N^{c-b}-1}{N^{c}-1} \nonumber\\
&\geq{} N^{-b} \cdot \log\left( N^{c-b} \right) - 2 \cdot \frac{N^{c-b}-1}{N^{c}-1} \nonumber\\
&\geq{} N^{-b} \cdot \log\left( N^{c-b} \right) - 2 \cdot \frac{N^{c-b}}{\frac{1}{2}N^{c}} \nonumber\\
&\geq{} N^{-b} \cdot \log\left( N^{c-b} \right) - 4 N^{-b}  \label{eq:appendix2}
\end{align}

Assuming that $N\geq 2^{8/(c-b)}$ we get

\begin{align*}
(\ref{eq:appendix2}){} 
& \geq {} \frac{1}{2}\cdot N^{-b} \cdot \log\left( N^{c-b} \right) \\
& \geq {} \frac{c-b}{2}\cdot N^{-b} \cdot \log\left( N \right).
\end{align*}

\end{proof}

\section{Privately Identifying a High-Sensitivity Function}
Let $S$ be a sample of $n$ i.i.d.\ elements from some distribution $\DDD$.
Recall that if a low-sensitivity function $f$ is identified by a differentially private algorithm operating on $S$, then w.h.p.\ $f(S)\approx f(\DDD^n)\triangleq\ex{S'\sim\DDD^n}{f(S')}$.
In this section we present a simple example showing that, in general, this is not the case for {\em high}-sensitivity functions. Specifically, we show that a differentially private algorithm operating on $S$ can identify a high-sensitivity function $f$ s.t.\ $|f(S)-f(\DDD^n)|$ is arbitrarily large, even though $|f(S')-f(\DDD^n)|$ is small for a fresh sample $S'\sim\DDD^n$.

\begin{theorem}
Fix $\beta,\eps,B>0$, let $\UUU$ be the uniform distribution over $X=\{\pm1\}^d$ where $d=\poly(1/\beta)$, and let $n\geq O(\frac{1}{\eps^2}\ln(1/\beta))$.
There exists an $(\eps,0)$-differentially private algorithm $\AAA$ 
that operates on a database $S\in (\{\pm1\}^d)^n$ and returns a function mapping $(\{\pm1\}^d)^n$ to $\R$, s.t.\ the following hold.
\begin{enumerate}
	\item For every $f$ in the range of $\AAA$ it holds that $\Pr_{S'\sim \UUU^n}[f(S')\neq f(\UUU^n)]\leq\beta$.
	\item $\Pr_{\substack{S\sim \UUU^n\\f\leftarrow\AAA(S)}}[|f(S)-f(\UUU^n)|\geq B]\geq1/2$.
\end{enumerate}
\end{theorem}

\begin{proof}

For $t\in[d]$, define $f_t:(\{\pm1\}^d)^n\rightarrow\R$ as 
$$
f_t(x_1,\dots,x_n) = 
\left\{ \begin{array}{ccl}
	0 & ,& \left|\sum_{i\in[n]} x_{i,t}\right| \leq \sqrt{2n\ln(2/\beta)}\\
	B & ,& \sum_{i\in[n]} x_{i,t} > \sqrt{2n\ln(2/\beta)} \\
	-B & ,& \sum_{i\in[n]} x_{i,t} < -\sqrt{2n\ln(2/\beta)} \\
\end{array} \right.
$$
That is, given a database $S$ of $n$ rows from $\{\pm1\}^d$, we define $f_t(S)$ as $0$ if the sum of column $t$ (in absolute value) is less than some threshold, and otherwise set $f_t(S)$ to be $\pm B$ (depending on the sign of the sum).
Observe that the global sensitivity of $f_t$ is $B$, and that $f_t(\UUU^n) \triangleq \ex{S'\sim \UUU^n}{f_t(S')}=0$. Also, by the Hoeffding bound, we have that
$$
\Pr_{S\sim \UUU^n}\left[f_t(S)\neq 0\right]\leq \beta.
$$

So, for every fixed $t$, with high probability over sampling $S\sim \UUU^n$ we have that $f_t(S)=0=f_t(\UUU^n)$. Nevertheless, as we now explain, if $d$ is large enough, then an $(\eps,0)$-differentially private algorithm can easily identify a ``bad'' index $t^*$ such that $|f_{t^*}(S)|=B$.

Consider the algorithm that on input $S=(x_1,x_2,\dots,x_n)$ samples an index $t\in[d]$ with probability proportional to $\exp\left( \frac{\eps}{4} \left|\sum_{i\in[n]} x_{i,t}\right| \right)$. We will call it algorithm \texttt{BadIndex}.

By the properties of the exponential mechanism, algorithm \texttt{BadIndex} is $(\eps,0)$-differentially private. Moreover, with probability at least $3/4$, the output $t^*$ satisfies
\begin{align}
\left|\sum_{i\in[n]} x_{i,t^*}\right| \;\;\geq\;\; \max_{t\in[d]}\left\{ \left|\sum_{i\in[n]} x_{i,t}\right| \right\} \;-\; \frac{4}{\eps}\ln\left(4d\right).\label{eq:overfit1}
\end{align}

In addition, by Theorem~\ref{thm:chernoffTight} (tightness of Chernoff bound), for every fixed $t$ it holds that
\begin{align*}
\Pr\left[ \sum_{i\in[n]} x_{i,t} \geq 1.11\cdot \sqrt{2n\ln(2/\beta)}  \right]\geq \left(\frac{\beta}{2}\right)^{45}.
\end{align*}

As the columns are independent, taking $d=2\left(\frac{2}{\beta}\right)^{45}$, we get that
\begin{align}
\Pr\left[ \max_{t\in[d]}\left\{\sum_{i\in[n]} x_{i,t} \right\}\geq 1.11\cdot \sqrt{2n\ln(2/\beta)}  \right]\geq 3/4.\label{eq:overfit2}
\end{align}

Combining~(\ref{eq:overfit1}) and~(\ref{eq:overfit2}) we get that with probability at least $1/2$ algorithm \texttt{BadIndex} identifies an index $t^*$ such that
$$
\left|\sum_{i\in[n]} x_{i,t^*}\right| \;\;\geq\;\; 1.11\cdot \sqrt{2n\ln(2/\beta)} \;-\; \frac{4}{\eps}\ln\left(4d\right).
$$
Assuming that $n\geq O(\frac{1}{\eps^2}\ln(1/\beta))$ we get that with probability at least $1/2$ algorithm \texttt{BadIndex} outputs an index $t^*$ s.t. $f_{t^*}(S)=B$.
\end{proof}

\subsection{Max-Information}

In this section we show that algorithm \texttt{BadIndex} has relatively high {\em max-information}:
Given two (correlated) random variables $Y$, $Z$, we use 
$Y\otimes Z$ denote the random variable obtained by drawing independent copies of $Y$ and $Z$ from their respective marginal distributions.

\begin{definition}[Max-Information~\cite{DworkFHPRR-nips-2015}]\label{def:maxinfo}
Let $Y$ and $Z$ be jointly distributed random variables over the domain $(\mathcal{Y},\mathcal{Z})$. 
The $\beta$-approximate max-information between $Y$ and $Z$ is defined as $$I_\infty^\beta (Y;Z) = \log \sup\limits_{\substack{\mathcal{O} \subseteq (\mathcal{Y} \times \mathcal{Z}),\\ \Pr[{ (Y,Z) \in \mathcal{O} }] > \beta}} \dfrac{\Pr[{(Y,Z) \in \mathcal{O}}] - \beta}{\Pr[{Y\otimes Z \in \mathcal{O}}] }.$$
An algorithm $\AAA: X^n \to F$ has $\beta$-approximate max-information of $k$ over product distributions, written $I^\beta_{\infty,P}(\AAA, n) \leq k$, if for every distribution $\DDD$ over $X$, we have $I^{\beta}_\infty(S; \AAA(S)) \leq k$ when $S\sim \DDD^n$.
\label{defn:maxinfo}
\end{definition}

It follows immediately from the definition  that approximate max-information controls the probability of ``bad events'' that can happen as a result of the dependence of $\AAA(S)$ on $S$: for every event $\cO$, we have  $\Pr[(S, \AAA(S)) \in \cO] \leq 2^k\Pr[S \otimes \AAA(S) \in \cO]+\beta$.\\

Consider again algorithm $\texttt{BadIndex}:(\{\pm1\})^n\rightarrow F$ that operates on database $S$ of size $n=O(\frac{1}{\eps^2}\ln(1/\beta))$ and identifies, with probability 1/2, a function $f$ s.t.\ $f(S)\neq 0$, even though $f(S')=0$ w.p.\ $1-\beta$ for a fresh sample $S'$.
Let us define $\cO$ as the set of all pairs $(S,f)$, where $S$ is a database and $f$ is a function in the range of algorithm \texttt{BadIndex} such that $f(S)\neq0$. That is,
$$
\cO = \left\{  (S,f)\in(\{\pm1\})^n \times F \; : \; f(S)\neq0 \right\}.
$$
If we assume that $I^{1/4}_{\infty,P}(\texttt{BadIndex}, n) \leq k$, then by Definition~\ref{def:maxinfo} we have:
$$
\frac{1}{2}\leq\Pr_{\substack{S\sim\UUU^n\\f\leftarrow\texttt{BadIndex}(S)}}[(S,f)\in\cO]\leq e^k\cdot\Pr_{\substack{S\sim\UUU^n\\T\sim\UUU^n\\f\leftarrow\texttt{BadIndex}(T)}}[(S,f)\in\cO] +\frac{1}{4}
\leq e^k\cdot\beta+\frac{1}{4}.
$$
So $k\geq\ln(\frac{1}{4\beta})=\Omega(\eps^2 n)$.

\bibliographystyle{plain}

\appendix

\section{Concentration Bounds Through Differential Privacy -- Missing Details}\label{sec:fullProof}

{
\renewcommand{\thetheorem}{\ref{claim:dpExpectation}}
\begin{claim}
Fix a function $f:X^n\rightarrow\R$ and parameters $\eps,\lambda\geq0$.
If $M:(X^n)^T\rightarrow Y$ is $(\eps,(f,\lambda))$-differentially private then for every $(f,\lambda)$-neighboring databases $\samples,\samples'\in(X^n)^T$ and every function $h:Y\rightarrow\R$ we have that
$$
\ex{y\leftarrow M(\vec{S})}{h(y)} \leq e^{-\eps}\cdot\ex{y\leftarrow M(\vec{S'})}{h(y)} \;\;+\;\; (e^\eps - e^{-\eps})\cdot \ex{y\leftarrow M(\vec{S'})}{|h(y)|}.
$$\end{claim}
\addtocounter{theorem}{-1}
}

\begin{proof}
\begin{align*}
\ex{y\leftarrow M(\vec{S})}{h(y)}
&=\int_0^\infty \Pr_{y\leftarrow M(\vec{S})}[h(y)\geq z]{\rm d}z \;\;-\;\; \int_{-\infty}^0 \Pr_{y\leftarrow M(\vec{S})}[h(y)\leq z]{\rm d}z\\
&\leq e^\eps\cdot \int_0^\infty \Pr_{y\leftarrow M(\vec{S'})}[h(y)\geq z]{\rm d}z \;\;-\;\; e^{-\eps}\cdot\int_{-\infty}^0 \Pr_{y\leftarrow M(\vec{S'})}[h(y)\leq z]{\rm d}z\\
&= e^{-\eps}\left[ \int_0^\infty \Pr_{y\leftarrow M(\vec{S'})}[h(y)\geq z]{\rm d}z \;\;-\;\; \int_{-\infty}^0 \Pr_{y\leftarrow M(\vec{S'})}[h(y)\leq z]{\rm d}z\right]\\
&+ (e^\eps - e^{-\eps})\cdot \int_0^\infty \Pr_{y\leftarrow M(\vec{S'})}[h(y)\geq z]{\rm d}z\\
&= e^{-\eps}\cdot\ex{y\leftarrow M(\vec{S'})}{h(y)} \;\;+\;\; (e^\eps - e^{-\eps})\cdot \int_0^\infty \Pr_{y\leftarrow M(\vec{S'})}[h(y)\geq z]{\rm d}z\\
&\leq e^{-\eps}\cdot\ex{y\leftarrow M(\vec{S'})}{h(y)} \;\;+\;\; (e^\eps - e^{-\eps})\cdot \int_0^\infty \Pr_{y\leftarrow M(\vec{S'})}[|h(y)|\geq z]{\rm d}z\\
&= e^{-\eps}\cdot\ex{y\leftarrow M(\vec{S'})}{h(y)} \;\;+\;\; (e^\eps - e^{-\eps})\cdot \ex{y\leftarrow M(\vec{S'})}{|h(y)|}
\end{align*}
\end{proof}

\subsection{Multi Sample Expectation Bound}

\begin{lem}[Expectation Bound] \label{lem:MKLCondExp}
Let $\DDD$ be a distribution over a domain $X$, let $\query:X^n\rightarrow\R$ , and let $\Delta,\lambda$
be s.t.\ for every $1\leq i\leq n$ it holds that
\begin{equation}
\ex{\substack{S\sim\DDD^{n}\\z\sim\DDD}}{\1\left\{\left|\query(S) - \query\left(S^{(i\leftarrow z)}\right)\right|>\lambda\right\} \cdot\left|\query(S) - \query\left(S^{(i\leftarrow z)}\right)\right| }\leq\Delta,
\label{eq:expectationMainCondition}
\end{equation}
where $S^{(i\leftarrow z)}$ is the same as $S$ except that the $i^{\text{th}}$ element is replaced with $z$.
Let $\AAA \from (\univ^{n})^{\trials} \to ([\trials]\cup\bot)$ be an $(\eps,(\query,\lambda))$-differentially private algorithm that operates on $T$ databases of size $n$ from $X$, and outputs an index $1\leq t \leq T$ or $\bot$. Then 
$$
\left|\ex{\substack{\samples\sim\dist^{nT} \\ \trial \leftarrow \AAA(\samples)}}{\1\{t\neq\bot\}\cdot(\query(\dist^n) - \query(\sample_{\trial}))} \right| \leq (e^\eps - e^{-\eps})\cdot \lambda n  + 6\Delta n T.
$$
If, in addition to~(\ref{eq:expectationMainCondition}), there exists a number $0\leq \tau \leq \lambda$ s.t.\ for every $1\leq i\leq n$ 
and every fixture of $S\in X^n$ we have that
\begin{equation}
\ex{\substack{y,z\sim\DDD}}{\1\left\{\left|\query(S^{(i\leftarrow y)}) - \query\left(S^{(i\leftarrow z)}\right)\right|\leq\lambda\right\} \cdot\left|\query(S^{(i\leftarrow y)}) - \query\left(S^{(i\leftarrow z)}\right)\right| }\leq\tau,
\label{eq:expectationSecondCondition}
\end{equation}
Then,
$$
\left|\ex{\substack{\samples\sim\dist^{nT} \\ \trial \leftarrow \AAA(\samples)}}{\1\{t\neq\bot\}\cdot(\query(\dist^n) - \query(\sample_{\trial}))} \right| \leq (e^\eps - e^{-\eps})\cdot \tau n  + 6\Delta n T.
$$
\end{lem}

We now present the proof assuming that~(\ref{eq:expectationSecondCondition}) holds for some $0\leq\tau\leq\lambda$. This is without loss of generality, as trivially it holds for $\tau=\lambda$.

\begin{proof}[Proof of Lemma~\ref{lem:MKLCondExp}]
Let $\samples' = (\sample'_1,\dots,\sample'_{\trials}) \sim \dist^{nT}$ be independent of $\samples$.  Recall that each element $\sample_{\trial}$ of $\samples$ is itself a vector $(\samp_{\trial, 1},\dots,\samp_{\trial, n}),$ and the same is true for each element $\sample'_{\trial}$ of $\samples'.$  We will sometimes refer to the vectors $\sample_{1},\dots,\sample_{\trials}$ as the \emph{subsamples of $\samples$.}

We define a sequence of intermediate samples that allow us to interpolate between $\samples$ and $\samples'$. Formally, for $\ell \in \{0,1,\dots,n\}$ define $\samples^{\ell} = (\sample^{\ell}_{1},\dots,\sample^{\ell}_{\trials}) \in (\univ^{n})^{\trials}$ where $\sample^{\ell}_{t}=(\samp^{\ell}_{\trial, 1},\dots,\samp^{\ell}_{\trial, n})$ and
$$
\samp^{\ell}_{\trial, i} = 
\left\{ \begin{array}{ccl}
	\samp_{\trial, i} & ,& i > \ell\\
	\samp'_{\trial, i} & ,& i \leq \ell \\
\end{array} \right.
$$
That is, every subsample $S^\ell_t$ of $\vec{S}^\ell$ is identical to $S'_t$ on the first $\ell$ elements, and identical to $S_t$ thereafter. 
By construction we have $\samples^0=\samples$ and $\samples^{n} = \samples'$.  Moreover, 
for every $t$ we have that $S^\ell_t$ and $S^{\ell-1}_t$ differ in exactly one element.
In terms of these intermediate samples we can write:

\begin{align}
&\left|\exx{\samples\sim\dist^{nT}}{\ex{\trial \leftarrow \AAA(\samples)}{\1\{t\neq\bot\}\cdot(\query(\dist^n) - \query(\sample_{\trial}))}} \right| \nonumber \\ 
&={}\left|\exx{\samples\sim\dist^{nT}}{\ex{\trial \leftarrow \AAA(\samples)}{\1\{t\neq\bot\}\cdot\left(\ex{\samples'\sim\dist^{nT}}{\query(\sample'_t)} - \query(\sample_{\trial})\right)}} \right| \nonumber \\ 
&={}\left|\exx{\samples\sim\dist^{nT}}{\exx{\trial \leftarrow \AAA(\samples)}\ex{\samples'\sim\dist^{nT}}{\1\{t\neq\bot\}\cdot\left(\query(\sample'_t) - \query(\sample_{\trial})\right)}} \right| \nonumber \\
&={} \left| \sum_{\ell \in [n]} \exx{\samples,\samples'\sim\dist^{nT}}{\ex{\trial \leftarrow \AAA(\samples)}{ \1\{t\neq\bot\}\cdot\left(\query(\sample^{\ell}_{\trial})- \query(\sample^{\ell-1}_{\trial}) \right)}} \right| \nonumber \\	 
&\leq{} \sum_{\ell \in [n]} \left| \exx{\samples,\samples'\sim\dist^{nT}}{\ex{\trial \leftarrow \AAA(\samples)}{ \1\{t\neq\bot\}\cdot\left(\query(\sample^{\ell}_{\trial})- \query(\sample^{\ell-1}_{\trial}) \right)}} \right| \nonumber\\ 
&={} \sum_{\ell \in [n]} \left| \exx{\samples,\samples'\sim\dist^{nT}}\exx{Z\sim\dist^T}{\ex{\trial \leftarrow \AAA(\samples)}{ \1\{t\neq\bot\}\cdot\left(\query(\sample^{\ell}_{\trial})- \query(\sample^{\ell-1}_{\trial}) \right)}} \right| \label{eq:1}
\end{align}

Given a multisample $\samples=(S_1,\dots,S_T)\in(X^n)^T$, a vector $Z=(z_1\dots,z_T)\in\univ^T$, and an index $1 \leq k \leq n$, we define $\samples^{(k\leftarrow Z)}$ to be the same as $\samples$ except that the $k^{\text{th}}$ element of {\em every} subsample $S_i$ is replaced with $z_i$. 
Observe that by construction, for every $\ell,Z$ we have $\samples^{\ell,(\ell\leftarrow Z)} = \samples^{\ell-1,(\ell\leftarrow Z)}$.
Thus,

\begin{align}
(\ref{eq:1}) &={} \sum_{\ell \in [n]} \left| \exx{\samples,\samples'\sim\dist^{nT}}\exx{Z\sim\dist^T}{\ex{\trial \leftarrow \AAA(\samples)}{ \1\{t\neq\bot\}\cdot\Bigg(\query(\sample^{\ell}_{\trial})- \query\left(\sample^{\ell,(\ell\leftarrow Z)}_{\trial}\right) \Bigg)
-
\1\{t\neq\bot\}\cdot\Bigg(\query(\sample^{\ell-1}_{\trial})- \query\left(\sample^{\ell-1,(\ell\leftarrow Z)}_{\trial}\right) \Bigg)
}} \right|. \label{eq:2}
\end{align}

Observer that the pairs $(\samples,\samples^{\ell})$ and $\left(\samples,\samples^{\ell,(\ell\leftarrow Z)}\right)$ are identically distributed. Namely, both $\samples^{\ell}$ and $\samples^{\ell,(\ell\leftarrow Z)}$ agree with $\samples$ on the last $(n-\ell)$ entries of every subsample, and otherwise contain i.i.d.\ samples from $\dist$. Hence, the expectation of
$\left(\query(\sample^{\ell}_{\trial})- \query\left(\sample^{\ell,(\ell\leftarrow Z)}_{\trial}\right)\right)$ is zero, and we get

\begin{align}
(\ref{eq:2}) &={} \sum_{\ell \in [n]} \left| \exx{\samples,\samples'\sim\dist^{nT}}\exx{Z\sim\dist^T}{\ex{\trial \leftarrow \AAA(\samples)}{ 
\1\{t\neq\bot\}\cdot\Bigg(
\query\left(\sample^{\ell-1,(\ell\leftarrow Z)}_{\trial}\right)  -
\query(\sample^{\ell-1}_{\trial})
\Bigg)
}} \right|. \label{eq:3}
\end{align}

Observer that the pair $(\samples^{\ell-1}, \samples)$ has the same distribution as $(\samples, \samples^{\ell-1}).$  Specifically, the first component is $n \trials$ independent samples from $\dist$ and the second component is equal to the first component with a subset of the entries replaced by fresh independent samples from $\dist$. Thus,

\begin{align}
(\ref{eq:3})&={} 
\sum_{\ell \in [n]} \left| \exx{\samples,\samples'\sim\dist^{nT}}\exx{Z\sim\dist^T}{\ex{\trial \leftarrow \AAA(\samples^{\ell-1})}{ 
\1\{t\neq\bot\}\cdot\Bigg(
\query\left(\sample^{(\ell\leftarrow Z)}_{\trial}\right)  -
\query(\sample_{\trial})
\Bigg)
}} \right|
\nonumber\\[1em]
&\leq{} \sum_{\ell \in [n]} \left| \exx{\samples,\samples'\sim\dist^{nT}}{\ex{Z\sim\dist^T}{
\1\left\{
\begin{array}{c}
	\max_{m\in[T]}|\query(\sample^{\ell-1}_{m}) - \query(\sample^{\ell}_{m})|\leq\lambda\\[0.5em]
	{\rm{and}}\\[0.5em]
	\max_{m\in[T]}|\query\left(\sample^{(\ell\leftarrow Z)}_{m}\right) - \query(\sample_{m})|\leq\lambda
\end{array} 
\right\}
 \cdot
\ex{\trial \leftarrow \AAA(\samples^{\ell-1})}{ 
\1\{t\neq\bot\}\cdot\Bigg(
\query\left(\sample^{(\ell\leftarrow Z)}_{\trial}\right)  -
\query(\sample_{\trial})
\Bigg)
} }} \right| \nonumber\\[1em]
&+{}\sum_{\ell \in [n]} \left| \exx{\samples,\samples'\sim\dist^{nT}}{\ex{Z\sim\dist^T}{
\1\left\{
\begin{array}{c}
	\max_{m\in[T]}|\query(\sample^{\ell-1}_{m}) - \query(\sample^{\ell}_{m})|>\lambda\\[0.5em]
	{\rm{or}}\\[0.5em]
	\max_{m\in[T]}|\query\left(\sample^{(\ell\leftarrow Z)}_{m}\right) - \query(\sample_{m})|>\lambda
\end{array} 
\right\}
\cdot
\max_{m\in[T]}\left|\query\left(\sample^{(\ell\leftarrow Z)}_{m}\right) - \query(\sample_{m})\right| }} \right| \label{eq:4}
\end{align}

When $\max_{m\in[T]}|\query(\sample^{\ell-1}_{m}) - \query(\sample^{\ell}_{m})|\leq\lambda$ we now use the properties of algorithm $\AAA$ to argue that $\AAA(\samples^{\ell-1}) \approx  \AAA(\samples^{\ell})$. Be Claim~\ref{claim:dpExpectation} we get that

\begin{align}
&(\ref{eq:4})\nonumber\\
&\leq{} \sum_{\ell \in [n]} \left| \exx{\samples,\samples'\sim\dist^{nT}}{\ex{Z\sim\dist^T}{
\1\left\{
\begin{array}{c}
	\max_{m\in[T]}|\query(\sample^{\ell-1}_{m}) - \query(\sample^{\ell}_{m})|\leq\lambda\\[0.5em]
	{\rm{and}}\\[0.5em]
	\max_{m\in[T]}|\query\left(\sample^{(\ell\leftarrow Z)}_{m}\right) - \query(\sample_{m})|\leq\lambda
\end{array} 
\right\}
 \cdot
\ex{\trial \leftarrow \AAA(\samples^{\ell})}{ 
\1\{t\neq\bot\}\cdot\Bigg(
\query\left(\sample^{(\ell\leftarrow Z)}_{\trial}\right)  -
\query(\sample_{\trial})
\Bigg)
} }} \right| \nonumber\\[1em]
&+{} (e^\eps - e^{-\eps})\cdot \sum_{\ell \in [n]} \left| \exx{\samples,\samples'\sim\dist^{nT}}{\ex{Z\sim\dist^T}{
\1\left\{
\begin{array}{c}
	\max_{m\in[T]}|\query(\sample^{\ell-1}_{m}) - \query(\sample^{\ell}_{m})|\leq\lambda\\[0.5em]
	{\rm{and}}\\[0.5em]
	\max_{m\in[T]}|\query\left(\sample^{(\ell\leftarrow Z)}_{m}\right) - \query(\sample_{m})|\leq\lambda
\end{array} 
\right\}
 \cdot
\ex{\trial \leftarrow \AAA(\samples^{\ell})}{ 
\1\{t\neq\bot\}\cdot\left|
\query\left(\sample^{(\ell\leftarrow Z)}_{\trial}\right)  -
\query(\sample_{\trial})
\right|
} }} \right| \nonumber\\[1em]
&+{}\sum_{\ell \in [n]} \left| \exx{\samples,\samples'\sim\dist^{nT}}{\ex{Z\sim\dist^T}{
\1\left\{
\begin{array}{c}
	\max_{m\in[T]}|\query(\sample^{\ell-1}_{m}) - \query(\sample^{\ell}_{m})|>\lambda\\[0.5em]
	{\rm{or}}\\[0.5em]
	\max_{m\in[T]}|\query\left(\sample^{(\ell\leftarrow Z)}_{m}\right) - \query(\sample_{m})|>\lambda
\end{array} 
\right\}
\cdot
\max_{m\in[T]}\left|\query\left(\sample^{(\ell\leftarrow Z)}_{m}\right) - \query(\sample_{m})\right| }} \right|\label{eq:5}
\end{align}

We can remove one of the two requirements in the indicator function in the middle row (this makes the expression bigger), to get:

\begin{align}
&(\ref{eq:5})\nonumber\\
&\leq{} \sum_{\ell \in [n]} \left| \exx{\samples,\samples'\sim\dist^{nT}}{\ex{Z\sim\dist^T}{
\1\left\{
\begin{array}{c}
	\max_{m\in[T]}|\query(\sample^{\ell-1}_{m}) - \query(\sample^{\ell}_{m})|\leq\lambda\\[0.5em]
	{\rm{and}}\\[0.5em]
	\max_{m\in[T]}|\query\left(\sample^{(\ell\leftarrow Z)}_{m}\right) - \query(\sample_{m})|\leq\lambda
\end{array} 
\right\}
 \cdot
\ex{\trial \leftarrow \AAA(\samples^{\ell})}{ 
\1\{t\neq\bot\}\cdot\Bigg(
\query\left(\sample^{(\ell\leftarrow Z)}_{\trial}\right)  -
\query(\sample_{\trial})
\Bigg)
} }} \right| \nonumber\\[1em]
&+{} (e^\eps - e^{-\eps})\cdot \sum_{\ell \in [n]} \left| \exx{\samples,\samples'\sim\dist^{nT}}{\exx{Z\sim\dist^T}{
\ex{\trial \leftarrow \AAA(\samples^{\ell})}{ 
\1\left\{\max_{m\in[T]}|\query\left(\sample^{(\ell\leftarrow Z)}_{m}\right) - \query(\sample_{m})|\leq\lambda 
\right\}
 \cdot
\1\{t\neq\bot\}\cdot\left|
\query\left(\sample^{(\ell\leftarrow Z)}_{\trial}\right)  -
\query(\sample_{\trial})
\right|
} }} \right| \nonumber\\[1em]
&+{}\sum_{\ell \in [n]} \left| \exx{\samples,\samples'\sim\dist^{nT}}{\ex{Z\sim\dist^T}{
\1\left\{
\begin{array}{c}
	\max_{m\in[T]}|\query(\sample^{\ell-1}_{m}) - \query(\sample^{\ell}_{m})|>\lambda\\[0.5em]
	{\rm{or}}\\[0.5em]
	\max_{m\in[T]}|\query\left(\sample^{(\ell\leftarrow Z)}_{m}\right) - \query(\sample_{m})|>\lambda
\end{array} 
\right\}
\cdot
\max_{m\in[T]}\left|\query\left(\sample^{(\ell\leftarrow Z)}_{m}\right) - \query(\sample_{m})\right| }} \right|
\label{eq:5b}
\end{align}

Furthermore, we can replace $\1\left\{\max_{m\in[T]}|\query\left(\sample^{(\ell\leftarrow Z)}_{m}\right) - \query(\sample_{m})|\leq\lambda\right\}$ in the middle row with the weaker requirement -- just for the specific $t$ that was selected by algorithm $\AAA$. This yields:

\begin{align}
&(\ref{eq:5b})\nonumber\\
&\leq{} \sum_{\ell \in [n]} \left| \exx{\samples,\samples'\sim\dist^{nT}}{\ex{Z\sim\dist^T}{
\1\left\{
\begin{array}{c}
	\max_{m\in[T]}|\query(\sample^{\ell-1}_{m}) - \query(\sample^{\ell}_{m})|\leq\lambda\\[0.5em]
	{\rm{and}}\\[0.5em]
	\max_{m\in[T]}|\query\left(\sample^{(\ell\leftarrow Z)}_{m}\right) - \query(\sample_{m})|\leq\lambda
\end{array} 
\right\}
 \cdot
\ex{\trial \leftarrow \AAA(\samples^{\ell})}{ 
\1\{t\neq\bot\}\cdot\Bigg(
\query\left(\sample^{(\ell\leftarrow Z)}_{\trial}\right)  -
\query(\sample_{\trial})
\Bigg)
} }} \right| \nonumber\\[1em]
&+{} (e^\eps - e^{-\eps})\cdot \sum_{\ell \in [n]} \left| \exx{\samples,\samples'\sim\dist^{nT}}{\exx{Z\sim\dist^T}{
\ex{\trial \leftarrow \AAA(\samples^{\ell})}{ 
\1\left\{|\query\left(\sample^{(\ell\leftarrow Z)}_{t}\right) - \query(\sample_{t})|\leq\lambda 
\right\}
 \cdot
\1\{t\neq\bot\}\cdot\left|
\query\left(\sample^{(\ell\leftarrow Z)}_{\trial}\right)  -
\query(\sample_{\trial})
\right|
} }} \right| \nonumber\\[1em]
&+{}\sum_{\ell \in [n]} \left| \exx{\samples,\samples'\sim\dist^{nT}}{\ex{Z\sim\dist^T}{
\1\left\{
\begin{array}{c}
	\max_{m\in[T]}|\query(\sample^{\ell-1}_{m}) - \query(\sample^{\ell}_{m})|>\lambda\\[0.5em]
	{\rm{or}}\\[0.5em]
	\max_{m\in[T]}|\query\left(\sample^{(\ell\leftarrow Z)}_{m}\right) - \query(\sample_{m})|>\lambda
\end{array} 
\right\}
\cdot
\max_{m\in[T]}\left|\query\left(\sample^{(\ell\leftarrow Z)}_{m}\right) - \query(\sample_{m})\right| }} \right|\label{eq:5c}
\end{align}

Using the fact that the pairs $(\vec{S},\vec{S}^\ell)$ and $(\vec{S}^\ell,\vec{S})$ are identically distributed, we can switch them in the middle row, to get

\begin{align}
&(\ref{eq:5c})\nonumber\\
&\leq{} \sum_{\ell \in [n]} \left| \exx{\samples,\samples'\sim\dist^{nT}}{\ex{Z\sim\dist^T}{
\1\left\{
\begin{array}{c}
	\max_{m\in[T]}|\query(\sample^{\ell-1}_{m}) - \query(\sample^{\ell}_{m})|\leq\lambda\\[0.5em]
	{\rm{and}}\\[0.5em]
	\max_{m\in[T]}|\query\left(\sample^{(\ell\leftarrow Z)}_{m}\right) - \query(\sample_{m})|\leq\lambda
\end{array} 
\right\}
 \cdot
\ex{\trial \leftarrow \AAA(\samples^{\ell})}{ 
\1\{t\neq\bot\}\cdot\Bigg(
\query\left(\sample^{(\ell\leftarrow Z)}_{\trial}\right)  -
\query(\sample_{\trial})
\Bigg)
} }} \right| \nonumber\\[1em]
&+{} (e^\eps - e^{-\eps})\cdot \sum_{\ell \in [n]} \left| \exx{\samples\sim\dist^{nT}}{\exx{t\leftarrow\AAA(\samples)}{
\ex{ \substack{ \samples'\sim\DDD^{nT}\\ Z\sim\dist^T } }{ 
\1\left\{|\query\left(\sample^{\ell,(\ell\leftarrow Z)}_{t}\right) - \query(\sample^{\ell}_{t})|\leq\lambda 
\right\}
 \cdot
\1\{t\neq\bot\}\cdot\left|
\query\left(\sample^{\ell,(\ell\leftarrow Z)}_{\trial}\right)  -
\query(\sample^{\ell}_{\trial})
\right|
} }} \right| \nonumber\\[1em]
&+{}\sum_{\ell \in [n]} \left| \exx{\samples,\samples'\sim\dist^{nT}}{\ex{Z\sim\dist^T}{
\1\left\{
\begin{array}{c}
	\max_{m\in[T]}|\query(\sample^{\ell-1}_{m}) - \query(\sample^{\ell}_{m})|>\lambda\\[0.5em]
	{\rm{or}}\\[0.5em]
	\max_{m\in[T]}|\query\left(\sample^{(\ell\leftarrow Z)}_{m}\right) - \query(\sample_{m})|>\lambda
\end{array} 
\right\}
\cdot
\max_{m\in[T]}\left|\query\left(\sample^{(\ell\leftarrow Z)}_{m}\right) - \query(\sample_{m})\right| }} \right|\label{eq:5d}
\end{align}

Using our assumptions on the function $f$ and the distribution $\DDD$ (for the middle row), brings us to:

\begin{align}
&(\ref{eq:5d})\nonumber\\
&\leq{} \sum_{\ell \in [n]} \left| \exx{\samples,\samples'\sim\dist^{nT}}{\ex{Z\sim\dist^T}{
\1\left\{
\begin{array}{c}
	\max_{m\in[T]}|\query(\sample^{\ell-1}_{m}) - \query(\sample^{\ell}_{m})|\leq\lambda\\[0.5em]
	{\rm{and}}\\[0.5em]
	\max_{m\in[T]}|\query\left(\sample^{(\ell\leftarrow Z)}_{m}\right) - \query(\sample_{m})|\leq\lambda
\end{array} 
\right\}
 \cdot
\ex{\trial \leftarrow \AAA(\samples^{\ell})}{ 
\1\{t\neq\bot\}\cdot\Bigg(
\query\left(\sample^{(\ell\leftarrow Z)}_{\trial}\right)  -
\query(\sample_{\trial})
\Bigg)
} }} \right| \nonumber\\
&+{} (e^\eps - e^{-\eps})n\tau\nonumber\\
&+{}\sum_{\ell \in [n]} \left| \exx{\samples,\samples'\sim\dist^{nT}}{\ex{Z\sim\dist^T}{
\1\left\{
\begin{array}{c}
	\max_{m\in[T]}|\query(\sample^{\ell-1}_{m}) - \query(\sample^{\ell}_{m})|>\lambda\\[0.5em]
	{\rm{or}}\\[0.5em]
	\max_{m\in[T]}|\query\left(\sample^{(\ell\leftarrow Z)}_{m}\right) - \query(\sample_{m})|>\lambda
\end{array} 
\right\}
\cdot
\max_{m\in[T]}\left|\query\left(\sample^{(\ell\leftarrow Z)}_{m}\right) - \query(\sample_{m})\right| }} \right|\label{eq:5e}
\end{align}

Our next task is to remove the indicator function in the first row. 
This is useful as the pairs $\left(\samples^{\ell},\samples^{(\ell\leftarrow Z)}\right)$ and $(\samples^{\ell},\samples)$ are identically distributed, and hence, if we were to remove the indicator function, the first row would be equal to zero.
To that end we add and subtract the first row with the complementary indicator function (this amounts to multiplying the third row by 2). We get

\begin{align}
(\ref{eq:5e})&\leq{} \sum_{\ell \in [n]} \left| \exx{\samples,\samples'\sim\dist^{nT}}{\ex{Z\sim\dist^T}{
\ex{\trial \leftarrow \AAA(\samples^{\ell})}{ 
\1\{t\neq\bot\}\cdot\Bigg(
\query\left(\sample^{(\ell\leftarrow Z)}_{\trial}\right)  -
\query(\sample_{\trial})
\Bigg)
} }} \right| \nonumber\\
&+{} (e^\eps - e^{-\eps})n\tau\nonumber\\
&+{}2\cdot\sum_{\ell \in [n]} \left| \exx{\samples,\samples'\sim\dist^{nT}}{\ex{Z\sim\dist^T}{
\1\left\{
\begin{array}{c}
	\max_{m\in[T]}|\query(\sample^{\ell-1}_{m}) - \query(\sample^{\ell}_{m})|>\lambda\\[0.5em]
	{\rm{or}}\\[0.5em]
	\max_{m\in[T]}|\query\left(\sample^{(\ell\leftarrow Z)}_{m}\right) - \query(\sample_{m})|>\lambda
\end{array} 
\right\}
\cdot
\max_{m\in[T]}\left|\query\left(\sample^{(\ell\leftarrow Z)}_{m}\right) - \query(\sample_{m})\right| }} \right|\label{eq:5f}
\end{align}

Now the first row is 0, so

\begin{align}
(\ref{eq:5f})&={} (e^\eps - e^{-\eps})n\tau\nonumber\\
&+{}2\cdot\sum_{\ell \in [n]} \left| \exx{\samples,\samples'\sim\dist^{nT}}{\ex{Z\sim\dist^T}{
\1\left\{
\begin{array}{c}
	\max_{m\in[T]}|\query(\sample^{\ell-1}_{m}) - \query(\sample^{\ell}_{m})|>\lambda\\[0.5em]
	{\rm{or}}\\[0.5em]
	\max_{m\in[T]}|\query\left(\sample^{(\ell\leftarrow Z)}_{m}\right) - \query(\sample_{m})|>\lambda
\end{array} 
\right\}
\cdot
\max_{m\in[T]}\left|\query\left(\sample^{(\ell\leftarrow Z)}_{m}\right) - \query(\sample_{m})\right| }} \right|
\label{eq:5g}
\end{align}

We can replace the {\em or} condition in the indicator function with the sum of the two conditions:

\begin{align}
(\ref{eq:5g})&\leq{} (e^\eps - e^{-\eps})n\tau\nonumber\\
&+{}2\cdot\sum_{\ell \in [n]} \left| \exx{\samples,\samples'\sim\dist^{nT}}{\ex{Z\sim\dist^T}{
\1\left\{
	\max_{m\in[T]}|\query(\sample^{\ell-1}_{m}) - \query(\sample^{\ell}_{m})|>\lambda
\right\}
\cdot
\max_{m\in[T]}\left|\query\left(\sample^{(\ell\leftarrow Z)}_{m}\right) - \query(\sample_{m})\right| }} \right|\nonumber\\
&+{}2\cdot\sum_{\ell \in [n]} \left| \exx{\samples,\samples'\sim\dist^{nT}}{\ex{Z\sim\dist^T}{
\1\left\{
	\max_{m\in[T]}|\query\left(\sample^{(\ell\leftarrow Z)}_{m}\right) - \query(\sample_{m})|>\lambda
\right\}
\cdot
\max_{m\in[T]}\left|\query\left(\sample^{(\ell\leftarrow Z)}_{m}\right) - \query(\sample_{m})\right| }} \right|
\label{eq:5h}
\end{align}

In the third row, we can replace $\max_{m\in[T]}$ with $\sum_{m\in[T]}$, to get

\begin{align}
(\ref{eq:5h})&\leq{} (e^\eps - e^{-\eps})n\tau\nonumber\\
&+{}2\cdot\sum_{\ell \in [n]} \left| \exx{\samples,\samples'\sim\dist^{nT}}{\ex{Z\sim\dist^T}{
\1\left\{
	\max_{m\in[T]}|\query(\sample^{\ell-1}_{m}) - \query(\sample^{\ell}_{m})|>\lambda
\right\}
\cdot
\max_{m\in[T]}\left|\query\left(\sample^{(\ell\leftarrow Z)}_{m}\right) - \query(\sample_{m})\right| }} \right|\nonumber\\
&+{}2\cdot\sum_{\ell \in [n]} \sum_{m\in[T]}\left| \exx{\samples,\samples'\sim\dist^{nT}}{\ex{Z\sim\dist^T}{
\1\left\{
	|\query\left(\sample^{(\ell\leftarrow Z)}_{m}\right) - \query(\sample_{m})|>\lambda
\right\}
\cdot
\left|\query\left(\sample^{(\ell\leftarrow Z)}_{m}\right) - \query(\sample_{m})\right| }} \right|
\label{eq:5i}
\end{align}

Applying our assumptions on $f$ and $\DDD$ to the third row brings us to

\begin{align}
(\ref{eq:5i})&\leq{} (e^\eps - e^{-\eps})n\tau\nonumber + 2nT\Delta\\
&+{}2\cdot\sum_{\ell \in [n]} \left| \exx{\samples,\samples'\sim\dist^{nT}}{\ex{Z\sim\dist^T}{
\1\left\{
	\max_{m\in[T]}|\query(\sample^{\ell-1}_{m}) - \query(\sample^{\ell}_{m})|>\lambda
\right\}
\cdot
\max_{m\in[T]}\left|\query\left(\sample^{(\ell\leftarrow Z)}_{m}\right) - \query(\sample_{m})\right| }} \right|\nonumber\\
\label{eq:5j}
\end{align}

The issue now is that the expression inside the indicator function is different from the expression outside of it. To that end, we split the indicator function as follows:

\begin{align}
(\ref{eq:5j})&\leq{} (e^\eps - e^{-\eps})n\tau\nonumber + 2nT\Delta\\
&+{}2\cdot\sum_{\ell \in [n]} \left| \exx{\samples,\samples'\sim\dist^{nT}}{\ex{Z\sim\dist^T}{
\1\left\{
\begin{array}{c}
	\max_{m\in[T]}|\query(\sample^{\ell-1}_{m}) - \query(\sample^{\ell}_{m})|>\lambda\\[0.5em]
	{\rm{and}}\\[0.5em]
	\max_{m\in[T]}\left|\query\left(\sample^{(\ell\leftarrow Z)}_{m}\right) - \query(\sample_{m})\right|>\lambda
\end{array} 
\right\}
\cdot
\max_{m\in[T]}\left|\query\left(\sample^{(\ell\leftarrow Z)}_{m}\right) - \query(\sample_{m})\right| }} \right|\nonumber\\[2em]
&+{}2\cdot\sum_{\ell \in [n]} \left| \exx{\samples,\samples'\sim\dist^{nT}}{\ex{Z\sim\dist^T}{
\1\left\{
\begin{array}{c}
	\max_{m\in[T]}|\query(\sample^{\ell-1}_{m}) - \query(\sample^{\ell}_{m})|>\lambda\\[0.5em]
	{\rm{and}}\\[0.5em]
	\max_{m\in[T]}\left|\query\left(\sample^{(\ell\leftarrow Z)}_{m}\right) - \query(\sample_{m})\right|\leq\lambda
\end{array} 
\right\}
\cdot
\max_{m\in[T]}\left|\query\left(\sample^{(\ell\leftarrow Z)}_{m}\right) - \query(\sample_{m})\right| }} \right|\nonumber\\[2em]
&\leq{} (e^\eps - e^{-\eps})n\tau\nonumber + 2nT\Delta\\
&+{}2\cdot\sum_{\ell \in [n]} \left| \exx{\samples,\samples'\sim\dist^{nT}}{\ex{Z\sim\dist^T}{
\1\left\{
	\max_{m\in[T]}\left|\query\left(\sample^{(\ell\leftarrow Z)}_{m}\right) - \query(\sample_{m})\right|>\lambda 
\right\}
\cdot
\max_{m\in[T]}\left|\query\left(\sample^{(\ell\leftarrow Z)}_{m}\right) - \query(\sample_{m})\right| }} \right|\nonumber\\[1em]
&+{}2\cdot\sum_{\ell \in [n]} \left| \exx{\samples,\samples'\sim\dist^{nT}}{\ex{Z\sim\dist^T}{
\1\left\{
	\max_{m\in[T]}|\query(\sample^{\ell-1}_{m}) - \query(\sample^{\ell}_{m})|>\lambda
\right\}
\cdot
\max_{m\in[T]}|\query(\sample^{\ell-1}_{m}) - \query(\sample^{\ell}_{m})| }} \right|\nonumber\\
&\leq{} (e^\eps - e^{-\eps})n\tau\nonumber + 6nT\Delta.\nonumber
\end{align}


\end{proof}

\subsection{Multi Sample Amplification}

\begin{theorem}[High Probability Bound] \label{thm:dpGeneralization}
Let $\DDD$ be a distribution over a domain $X$, let $\query:X^n\rightarrow\R$ , and let $\Delta,\lambda,\tau$
be s.t.\ for every $1\leq i\leq n$ it holds that
$$
\ex{\substack{S\sim\DDD^{n}\\z\sim\DDD}}{\1\left\{\left|\query(S) - \query\left(S^{(i\leftarrow z)}\right)\right|>\lambda\right\} \cdot\left|\query(S) - \query\left(S^{(i\leftarrow z)}\right)\right| }\leq\Delta,
$$
and, furthermore, $\forall S\in X^n$ and $\forall 1\leq i\leq n$ we have
$$
\ex{\substack{y,z\sim\DDD}}{\1\left\{\left|\query(S^{(i\leftarrow y)}) - \query\left(S^{(i\leftarrow z)}\right)\right|\leq\lambda\right\} \cdot\left|\query(S^{(i\leftarrow y)}) - \query\left(S^{(i\leftarrow z)}\right)\right| }\leq\tau,
$$
where $S^{(i\leftarrow z)}$ is the same as $S$ except that the $i^{\text{th}}$ element is replaced with $z$.
Then for every $\eps>0$ we have that
$$
\Pr_{S\sim\DDD^n}\left[ | f(S) - f(\DDD^n) | \geq 6(e^\eps - e^{-\eps})\tau n \right] < \frac{14\Delta}{(e^\eps - e^{-\eps})\tau},
$$
provided that $n\geq O\left(\frac{\lambda}{\eps(e^\eps - e^{-\eps})\tau}\log\left(\frac{(e^\eps - e^{-\eps}) \tau}{\Delta}\right)\right)$
\end{theorem}

\begin{proof}
We only analyze the probability that $(f(S)-f(\DDD^n))$ is large. The analysis for $(f(\DDD^n)-f(S))$ is symmetric.
Assume towards contradiction that with probability at least $\frac{7\Delta}{(e^\eps - e^{-\eps})\tau}$ we have that $ f(S) - f(\DDD^n)  \geq 6(e^\eps - e^{-\eps})\tau n$. We now construct the following algorithm $\BBB$ that contradicts our expectation bound.

\begin{algorithm}[H]
\caption{$\BBB$}\addcontentsline{lof}{figure}{Algorithm $\BBB$}
\vspace{2pt}
{\bf Input:} $T$ databases of size $n$ each: $\vec{S}=(S_1,\dots,S_T)$, where $T\triangleq\left\lfloor \frac{(e^\eps - e^{-\eps})\tau}{7\Delta }\right\rfloor$.
\begin{enumerate}[rightmargin=10pt,itemsep=1pt,topsep=4pt]

\item Set $H=\{\bot,1,2,\dots,T\}$.

\item For $i=1,...,T$, define $q(\vec{S},i) = f(S_i) - f(\DDD^n)$. Also set $q(\vec{S},\bot)=0$.

\item Sample $t^*\in H$ with probability proportional to $\exp\left(\frac{\eps}{2\lambda} q(\vec{S},t)\right)$.

\end{enumerate}
\textbf{Output:} $t.$
\end{algorithm}

The fact that algorithm $\BBB$ is $(\eps,(f,\lambda))$-differentially private follows from the standard analysis of the Exponential Mechanism of McSherry and Talwar~\cite{McSherryT07}. The proof appears in Claim~\ref{claim:ExpMechPrivacy} for completeness.

Now consider applying $\BBB$ on databases $\vec{S} = (S_1,\dots,S_T)$ containing i.i.d.\ samples from $\DDD$. By our assumption on $\DDD$ and $f$, for every $t$ we have that 
$f(S_t) - f(\DDD^n) \geq 6(e^\eps - e^{-\eps})\tau n$
 with probability at least $\frac{7\Delta}{(e^\eps - e^{-\eps})\tau}$. By our choice of $T = \left\lfloor \frac{(e^\eps - e^{-\eps})\tau}{7\Delta} \right\rfloor$, we therefore get
$$\Pr_{\vec{S}\sim\DDD^{nT}}\left[{\max_{t \in [T]}  \left\{   f(S_t)- f(\DDD^n)  \right\} \geq 6(e^\eps - e^{-\eps})\tau n }\right] \geq 1 - \left( 1 - \frac{7\Delta}{(e^\eps - e^{-\eps})\tau} \right)^T \geq \frac12.$$
The probability is taken over the random choice of
the examples in $\vec{S}$ according to $\DDD$.
Thus, by Markov's inequality,
\begin{equation}\label{eq:LargeError}
\E_{\vec{S}\sim\DDD^{nT}}\left[\max_{t \in H} \left\{ q(\vec{S},t) \right\}\right] =
\E_{\vec{S}\sim\DDD^{nT}}\left[\max\left\{0\;,\; \max_{t \in [T]}  \left(f(S_t)- f(\DDD)\right) \right\}\right] \geq 3(e^\eps - e^{-\eps})\tau n.
\end{equation}

So, in expectation, $\max_{t \in H}  \left(q(\vec{S},t)\right)$ is large. In order to contradict the expectation bound of Theorem~\ref{thm:dpGeneralization}, we need to show that this is also the case for the index $t^*$ that is sampled on Step~3. To that end, we now use the following technical claim, stating that the expected quality of a solution sampled as in Step~3 is high.

\begin{claim}[e.g.,~\cite{BassilyNSSSU16}] \label{claim:EMutility_duplicate}
Let $H$ be a finite set, $h : H \to \mathbb{R}$ a function, and $\eta >0$. Define a random variable $Y$ on $H$ by $\Pr[Y=y] = \exp(\eta h(y))/C$, where $C= \sum_{y \in H} \exp(\eta h(y))$. Then $\ex{}{h(Y)} \geq \max_{y \in H} h(y) - \frac{1}{\eta}\ln |H|$.
\end{claim}

For every fixture of $\vec{S}$, we can apply Claim~\ref{claim:EMutility_duplicate} with $h(t) =  q(\vec{S},t)$ and $\eta = \frac{\eps}{2\lambda}$ to get 
\begin{equation*}
\E_{t^*\in_R H}[q(\vec{S},t^*)]
=\E_{t^*\in_R H}\Big[\1{\{t^*\neq\bot\}}\cdot\left(f(S_{t^*})-f(\DDD^n)\right)\}\Big] 
 \geq \max\{0\;,\;\max_{t\in [T]}(f(S_t)-f(\DDD^n))\} - \frac{2\lambda}{\eps} \ln(T+1).
\end{equation*}
Taking the expectation also over $\vec{S}\sim\DDD^{nT}$ we get that
\begin{eqnarray*}
\E_{\substack{\vec{S}\sim\DDD^{nT} \\ t^*\leftarrow\BBB\left(\vec{S}\right)}}\Big[\1{\{t^*\neq\bot\}}\cdot\left(f(S_{t^*})-f(\DDD^n)\right)\}\Big] 
&\geq& \E_{\vec{S}\sim\DDD^{nT}}\left[\max\left\{0\;,\; \max_{t \in [T]}  \left(f(S_t)- f(\DDD^n)\right) \right\}\right] - \frac{2\lambda}{\eps} \ln(T+1)\\
&\geq& 3(e^\eps - e^{-\eps})\tau n - \frac{2\lambda}{\eps} \ln(T+1).
\end{eqnarray*}
This contradicts Theorem~\ref{thm:dpGeneralization} whenever $n>\frac{2\lambda}{\eps(e^\eps - e^{-\eps}) \tau}\ln(T+1)=\frac{2\lambda}{\eps(e^\eps - e^{-\eps}) \tau}\ln(\frac{(e^\eps - e^{-\eps})\tau}{7\Delta}+1)$.
\end{proof}

\begin{claim}\label{claim:ExpMechPrivacy}
Algorithm $\BBB$ is $(\eps,(f,\lambda))$-differentially private.
\end{claim}

\begin{proof}
Fix two $(f,\lambda)$-neighboring databases $\vec{S}$ and $\vec{S'}$, and let $b\in\{\bot,1,2,\dots,T\}$ be a possible output.
We have that
\begin{align}
	\Pr[\BBB(\vec{S})=b]&={}\frac{\exp(\frac{\eps}{2\lambda}\cdot q(\vec{S},b))}{\sum_{a\in H}\exp(\frac{\eps}{2\lambda}\cdot q(\vec{S},a))}\label{eq:11}
\end{align}

Using the fact that $\vec{S}$ and $\vec{S'}$ are $(f,\lambda)$-neighboring, for every $a\in H$ we get that $q(\vec{S'},a)-\lambda\leq q(\vec{S},a)\leq q(\vec{S'},a)+\lambda$. Hence,

\begin{align*}
	(\ref{eq:11})&\leq{} \frac{\exp(\frac{\eps}{2\lambda}\cdot [q(\vec{S'},b)+\lambda])}{\sum_{a\in H}\exp(\frac{\eps}{2\lambda}\cdot [q(\vec{S'},a)-\lambda])}\\
&={} \frac{e^{\eps/2}\cdot\exp(\frac{\eps}{2\lambda}\cdot q(\vec{S'},b))}{e^{-\eps/2}\sum_{a\in H}\exp(\frac{\eps}{2\lambda}\cdot q(\vec{S'},a))}\\
&={} e^{\eps}\cdot\Pr[\BBB(\vec{S'})=b].
\end{align*}

For any possible {\em set} of outputs $B\subseteq\{\bot,1,2,\dots,T\}$ we now have that

$$
\Pr[\BBB(\vec{S})\in B]=\sum_{b\in B}\Pr[\BBB(\vec{S})=b]\leq \sum_{b\in B}e^{\eps}\cdot\Pr[\BBB(\vec{S'})=b] = \Pr[\BBB(\vec{S'})\in B].
$$
\end{proof}

\end{document}